\documentclass[hidelinks,onefignum,onetabnum]{siamart220329}

\definecolor{darkred}{rgb}{.7,0,0}
\definecolor{darkgreen}{rgb}{0,0.7,0}

\def\cG{{\mathcal G}}
\def\cH{{\mathcal H}}

\def\cO{{\mathcal O}}
\def\cP{{\mathcal P}}

\def\bR{{\mathbb R}}

\def\bE{{\mathbb E}}
\def\bP{{\mathbb P}}

\def\NPDF{{\mathcal N}}

\def\f0{{\mathbf 0}}

\def\KL{{\operatorname{KL}}}

\def\cO{\mathcal{O}}

\newcommand{\dd}{\mathsf{d}}
\newcommand{\sJ}{\mathsf{J}}
\newcommand{\sI}{\mathsf{I}}

\usepackage{subfigure}
\usepackage[noend]{algorithmic}

\usepackage{lipsum}
\usepackage{amsfonts}
\usepackage{graphicx}
\usepackage{epstopdf}
\usepackage{algorithmic}
\ifpdf
  \DeclareGraphicsExtensions{.eps,.pdf,.png,.jpg}
\else
  \DeclareGraphicsExtensions{.eps}
\fi

\newsiamremark{remark}{Remark}
\newsiamremark{hypothesis}{Hypothesis}
\crefname{hypothesis}{Hypothesis}{Hypotheses}
\newsiamthm{claim}{Claim}

\title{Efficient Prior Calibration From Indirect Data\thanks{
\funding{
MG is supported by a Royal Academy of Engineering Research
Chair and EPSRC grants EP/X037770/1, EP/Y028805/1, EP/W005816/1, EP/V056522/1 , EP/V056441/1, EP/T000414/1, and EP/R034710/1.
AMS is grateful for support through a Department of Defense Vannevar Bush Faculty Fellowship, from the Air Force Office of Scientific Research under MURI award number FA9550-20-1-0358 (Machine Learning and Physics-Based Modeling and Simulation) and by the SciAI Center, funded by the Office of Naval Research (ONR), under Grant Number N00014-23-1-2729.
AV is supported through the EPSRC ROSEHIPS grant  EP/W005816/1.
}
}}

\headers{Efficient Prior Calibration From Indirect Data}{O. D. Akyildiz, M. Girolami, A. M. Stuart, A. Vadeboncoeur}

\author{
O. Deniz Akyildiz\thanks{Department of Mathematics, Imperial College London, London, SW7 2AZ  UK (\email{deniz.akyildiz@imperial.ac.uk}).}
\and Mark Girolami\thanks{The  Alan  Turing  Institute, London, NW1 2DB UK,  and Department of Engineering, University of Cambridge, Cambridge, CB2 1PZ UK (\email{mag92@cam.ac.uk}).}
\and Andrew M. Stuart\thanks{Department of Computing and Mathematical Sciences, California Institute of Technology, Pasadena, CA 91125 USA (\email{astuart@caltech.edu}).}
\and Arnaud Vadeboncoeur\thanks{Department of Engineering, University of Cambridge,  Cambridge, CB2 1PZ UK (\email{av537@cam.ac.uk}). Corresponding author.}
}

\usepackage{amsopn}

\ifpdf
\hypersetup{
  pdftitle={Efficient Prior Calibration From Indirect Data},
  pdfauthor={O. D. Akyildiz, M. Girolami, A. Stuart, A. Vadeboncoeur}
}
\fi

\def \r{\mathbf r}

\def \I{\mathbf I}
\def \0{\mathbf 0}

\def\md{{\mathrm d}}
\def\KL{{\mathrm{KL}}}
\def\cZ{\mathcal{Z}}

\def\cH{\mathcal{H}}
\def\Sb{\mathbb S}
\def\b0{\mathbf 0}

\def \Rb {\mathbb R}
\def \Eb {\mathbb E}

\def \Sb{\mathbb S}

\def \B{\mathrm{B}}

\def \r{\mathsf r}
\def \I {\mathrm I}
\usepackage{xcolor}
\usepackage{bbm}

\usepackage{tikz}
\usetikzlibrary{fit,positioning,bayesnet}
\usetikzlibrary{automata,arrows}
\usetikzlibrary{backgrounds}

\def\genbox#1#2#3#4#5#6{
    \leavevmode\raise#4bp\hbox to#5bp{\vrule height#5bp depth0bp width0bp
    \pdfliteral{q .5 w \csname #2COLOR\endcsname\space RG
                       \csname #3PDF\endcsname{#5}{#6} S Q
             \ifx1#1 q \csname #2COLOR\endcsname\space rg 
                       \csname #3PDF\endcsname{#5}{#6} f Q\fi}\hss}}

\def\diabox     #1#2{\genbox{#1}{#2}  {dia}      {-.5} {6}    {3}}

\definecolor{darkmagenta}{rgb}{0.55, 0.0, 0.55}

\begin{document}

\maketitle

\begin{abstract}
Bayesian inversion is central to the quantification of uncertainty
within problems arising from numerous applications in science and engineering. 
To formulate the approach, four ingredients are required:  a \emph{forward model} 
mapping the unknown parameter to an element of a solution space, often the solution
space for a differential equation; an \emph{observation operator} mapping an element 
of the solution space to the data space; a \emph{noise model} describing how noise pollutes 
the observations; and a \emph{prior model} describing
knowledge about the unknown parameter before the data is acquired. This paper is concerned 
with learning the prior model from data; in particular, learning the prior from multiple 
realizations of indirect data obtained through the noisy observation process.
The prior is represented, using a generative model, as the pushforward of a Gaussian in a latent
space; the pushforward map is learned by minimizing an appropriate loss function.
A metric that is well-defined under empirical approximation is used to 
define the loss function for the pushforward map to make an implementable
methodology. Furthermore, an efficient residual-based
neural operator approximation of the forward model is proposed and it is shown that
this may be learned concurrently with the pushforward map, using a bilevel optimization 
formulation of the problem; this use of neural operator approximation has the potential to make prior learning from indirect data more computationally efficient,
especially when the observation process is expensive, non-smooth or not known.
The ideas are illustrated with the Darcy flow inverse problem of finding permeability from piezometric head measurements.
\end{abstract}

\begin{keywords}
inverse problems, generative models, prior learning, operator learning, differential equations
\end{keywords}

\begin{MSCcodes}
68T37, 65N30, 62F15, 35R30
\end{MSCcodes}

\section{Introduction}
\label{sec:I}

\subsection{Set-Up}
\label{ssec:SU}
This paper is concerned with learning a generative model for unobserved $\{z^{(n)}\}_{n=1}^N$ from
indirect and noisy data $\{y^{(n)}\}_{n=1}^N$ given by
\begin{align}
\label{eq:1}
y^{(n)} = \cG(z^{(n)}) + \varepsilon^{(n)}_\eta, 
\end{align}
where the noise $\varepsilon^{(n)}_\eta \sim \eta$ i.i.d.\,.
The setting where $\cG(\cdot)$ is the identity and the noise is zero is the standard problem
of generative modeling and is well studied; in the context of this paper, Bayesian inversion,
such a generative model may be used to construct a prior measure from prior samples $\{z^{(n)}\}_{n=1}^N$. However, in many science and engineering applications the prior samples $\{z^{(n)}\}_{n=1}^N$ are not directly
observed, but the  $\{y^{(n)}\}_{n=1}^N$, arising from \eqref{eq:1} with choice $\cG=g \circ F^\dagger$ originating from multiple instantiations of physical systems
are available. Thus, we are interested in the setting where
\begin{align}
\label{eq:2}
y^{(n)} = g \circ F^\dagger(z^{(n)}) + \varepsilon^{(n)}_\eta.
\end{align}
Here, $F^\dagger: Z \mapsto U$ maps between function spaces $Z,U$, representing a 
PDE parameter-to-solution operator, and $g: U \mapsto \bR^{d_y}$ is a solution-to-data map. From the resulting finite dimensional data we wish to construct a generative model for a prior measure on a function space, giving rise to the unobserved prior samples 
$\{z^{(n)}\}_{n=1}^N$.

To overview our approach to this problem we first describe it at a population loss level.
Let $\nu$ denote the law of the $y^{(n)}$ and $\mu$ the desired law of the $z^{(n)}$. Letting
$*$ denote convolution of measures and $\#$ denote pushforward, introducing divergence $\dd_1$
between probability measures in data space and $\cH(\cdot)$ the regularization term on measures 
in the input space, we define population loss
\begin{align}\label{eq:ideal_loss}
\sJ_1(\mu) = \dd_1\Bigl(\nu, \eta * (g \circ F^\dagger)_\# \mu \Bigr) + {\cH}(\mu). \tag{{Functional 1}}
\end{align}
To develop algorithms to find $\mu$ we will represent $\mu$ as pushforward under map $T^\alpha$
of Gaussian measure $\mu_0$ on a latent space.\footnote{Other generative models replacing the Gaussian with different,
but also straightforward to sample, measures can easily be accommodated; we choose a Gaussian in the latent
space to make the presentation explicit.} Here $\alpha \in \bR^{d_\alpha}$ represents a finite dimensional parameterization
of the pushforward from $\mu_0$ to $\mu$. We replace the regularization term $\cH(\cdot)$ on measure $\mu$ 
by regularization term $h: \bR^{d_\alpha} \to \bR$ on $\alpha.$ We then consider the modified population loss
\begin{align}\label{eq:ideal_loss2}
\sJ_2(\alpha) = \dd_1\Bigl(\nu, \eta * (g \circ F^\dagger \circ T^\alpha)_\# \mu_0\Bigr) + h(\alpha). \tag{{Functional 2}}
\end{align}
We also observe that $F^\dagger$ may be expensive to compute and it may be desirable to replace it
by a neural operator $F^\phi$ whose parameters $\phi \in \bR^{d_\phi}$ need to be learned so that
$F^\phi \approx F^\dagger$ in sets of high probability under $\mu.$ But we do not know $\mu$, indeed we
are trying to find it; thus the optimal parameters $\phi$ will depend on $\alpha$ and
be defined by $\phi=\phi^*(\alpha)$. We thus introduce loss function
\begin{align}\label{eq:ideal_loss3}
\sJ_3(\alpha) = \dd_1\Bigl(\nu, \eta * (g \circ F^{\phi^*(\alpha)} \circ T^\alpha)_\# \mu_0\Bigr) + h(\alpha). \tag{{Functional 3}}
\end{align}
At the heart of all these loss functions is a matching of distributions.
In practice both $\nu$ and the pushforward of $\mu_0$ will only be available empirically and so it is necessary that the divergence $\dd_1$ can be readily evaluated on empirical measures. Empiricalized versions of the functionals $\sJ_2, \sJ_3$ will form the basis of computational methodology proposed in this paper.  
The mapping $\phi^*(\alpha)$ will also be
learned using minimization of a loss function, involving matching of
distributions and evaluated empirically. The functional $\sJ_1$ provides a theoretical underpinning of our approach; and in the case $N=1$
will be linked, in the empirical setting, to Bayes Theorem.
The efficiency of the proposed methodology is due to: (i) the replacement of costly PDE simulations with evaluations of a concurrently learned surrogate model trained through readily computable PDE residuals; (ii) as will be mentioned in Remark~\ref{rem:scalingVBayes}, the proposed method scales more favorably than alternative Bayesian approaches due to the challenges of high dimensional posterior sampling. 

Subsection \ref{ssec:C}, which follows, summarizes our contributions and
outlines the paper. In Subsection \ref{ssec:LR} we review relevant literature in the area and in Subsection \ref{ssec:N} we overview the notation and model problem (Darcy) used in the paper.

\subsection{Contributions and Outline}
\label{ssec:C}
The proposed novel methodology allows the construction of a calibrated measure over the parameters underlying a PDE model, given data from a \textit{collection} of physical systems. Furthermore, the methodology may be combined with a novel concurrent neural operator approximation of the PDE. Together these ideas hold the potential to improve the accuracy and efficiency of Bayesian inversion and of generative modeling for physical systems. The work can be broken down into five primary contributions:
\begin{enumerate}
    \item We introduce a suitable choice for divergence $\dd_1$ based on sliced Wasserstein-2 distance
    and demonstrate that it leads to computationally feasible objective $\sJ_2$  \eqref{eq:ideal_loss2}.
    \item We introduce a residual-based probabilistic loss function to define choice of parameters
    $\phi^*(\alpha)$ in the neural operator approximation.
    \item  With this definition of $\phi^*(\alpha)$ we demonstrate a computationally feasible objective $\sJ_3$  \eqref{eq:ideal_loss3}.
    \item We show that, with our choice of $\dd_1$, minimization of
    \ref{eq:ideal_loss} may be linked to Bayes Theorem when $N=1$.
    \item In order to be concrete we describe our methodology in the context of the Darcy flow
    model of porous medium flow which may be viewed as a mapping from the permeability field ($z$) to 
    linear functionals of the piezometric head ($y$). Numerical experiments with Darcy flow, for two different choices of pushforward families $T^\alpha$,
    are used to demonstrate feasibility and consistency of the proposed methodology.
\end{enumerate}

In Section \ref{sec:R} we describe the efficient residual-based approach to operator learning that we
adopt in this paper, addressing contribution 2. Section \ref{sec:M} introduces specific divergences for definition of Functionals 1--3, and for the residual-based learning, addressing contributions 1 and 3;
we also address contribution 4 in Theorem \ref{t:T}.
Section \ref{sec:A} discusses algorithmic details, giving further detail on contributions 1 and 3 whilst
Section \ref{sec:N} implements the algorithms on the Darcy flow problem, contribution 5. 
We conclude and discuss future works in Section \ref{sec:C}.

\subsection{Literature Review}
\label{ssec:LR}
We overview relevant literature, First, we discuss the learning of priors from data.
Secondly we describe the surrogate modeling literature. And finally, we discuss related transport-based inference methods.

\subsubsection{Learning Priors}
The task of selecting \textit{a best prior} has received much attention in the Bayesian statistics community, with many objectives and ideals in mind~\cite{jaynes1968prior}. Certain efforts concentrate on the careful formulation of uninformative priors~\cite{jeffreys1946invariant, bernardo1979reference, bernardo2009bayesian}. Others focus on mathematical tractability through conjugacy~\cite{raiffa2000applied} or eliciting priors from domain experts~\cite{o2019expert}. Some see the prior as an opportunity to share information from observations with possibly different underlying parameters, that are assumed to be drawn from the same distribution. Taking this point of view, are methods related to hierarchical Bayes~\cite{gelman1995bayesian}, empirical Bayes~\cite{robbins1964empirical, robbins1992empirical} and  parametric empirical Bayes~\cite{casella1985introduction, morris1983parametric}. In many ways, the idea of using a set of data to explicitly target an unconditional distribution is the basis of many modern generative modeling methods in the field of machine learning (ML). These include score-based models~\cite{song2020score}, diffusion models~\cite{song2019generative, ho2020denoising}, 
variational auto-encoders~\cite{kingma2013auto}, energy-based models ~\cite{song2021train, teh2003energy}, normalizing flows~\cite{papamakarios2021normalizing}, gradient-flows~\cite{bonnotte2013unidimensional}, and more.  Data-based unconditional distributions like these have been shown to be not only expressive generative models, but also powerful priors for Bayesian inversion~\cite{feng2023score, boys2023tweedie}. 
These ideas appear in a number of recent papers focussed on inverse-problems
\cite{patel2021gan,asim2020invertible,arridge2019solving}. A key difference in our work and the basis for our contribution is that, while above works focus on learning (and using) priors from which unconditional samples are available, our methodology here extends this idea to learn priors from indirect data. This idea of indirect knowledge of priors and data
is exploited in \cite{daras2024ambient} and also in \cite{gao2023image} for linear operators
where the set of Bayesian inverse problems defined individually by
\eqref{eq:1}, for each $n$, is solved in the case where $\cG(\cdot)=A\cdot$ for some linear
operator $A$; the collection of inverse problems is used to both learn prior information and learn the
dependence of each individual posterior on data.

\subsubsection{Surrogate Modeling and Operator Learning}
In many engineering applications related to design optimization, parameter inversion, and forward uncertainty quantification tasks, it is necessary to evaluate numerical models many times. 
When using a classical numerical scheme, there is no information carry-over from one numerical solve operation to the next.
Hence, there is room for improvement in the form of somehow interpolating information from the one numerical solution to the next. The idea of replacing computer code with a cheap-to-evaluate statistical interpolation model is well explored~\cite{sacks1989design,kennedy2001bayesian}. Such methods have found their place as viable model order reduction techniques for multi-query problems.
Modern advances have now begun to pose the task of surrogate modeling directly in function-space, resulting in the field of operator learning~\cite{bhattacharya2021model,lu2021learning, kovachki2023neural, li2020fourier}. These methods are based on gathering input-output datasets of PDE parameters and PDE solutions obtain via classical numerical schemes, such as finite element models (FEM) and spectral methods. 
Physics informed surrogate modeling attempts to directly incorporate PDE information into the learning task~\cite{rixner2021probabilistic, zhu2019physics, li2021physics}. The methods may then be data-less or semi-data-informed. 
Like PINO \cite{li2021physics} our optimization objective balances a data-driven loss with a PDE-enforcing regularizer. Unlike PINO, however, our data-driven loss is not related to operator learning; operator learning is introduced purely through the PDE-enforcing regularizer.  Also, unlike PINO, which adds the loss and the regularizer, we adopt a bilevel optimization strategy because of issues related to balancing the 
two terms. 
Finally  we use the variational form of the PDE to define the regularizer, demonstrating that the FNO operator interacts well with FEM-interpolation to enable derivative-based optimization.
Certain classes of methods also pose prior distributions over PDE parameters and attempt to learn a surrogate trained on random draws from that prior~\cite{vadeboncoeur2023fully, vadeboncoeur2023random}.

\subsubsection{Pushforwards and Wasserstein Losses}
Minimization of regularized loss functions over the space of probability measures, as exemplified by \ref{eq:ideal_loss},
is central to modern computational statistics and machine learning: it lies at the heart of
variational inference \cite{wainwright2008graphical} and in many other emerging inference
problems for probability measures \cite{crucinio2024solving,liutkus2019sliced, li2024differential, li2024stochastic, li2025inverse}.
In statistical inference and ML, there are many learning objective functions to make use of. A well known method is maximizing the marginal likelihood. However, the task we are interested in for this work is distributional learning. Hence, we must use a statistical divergence between measures. A common choice is the KL divergence. However this is not useful when both measures being compared are empirical, due to the non-overlapping support of sampled Diracs, as is the case in this paper (the $D_\KL$ is either $0$ or $\infty$).

Noting this restriction on suitable divergences we proceed to identify appropriate choices.
The (MMD)   maximum mean divergence~\cite{smola2007hilbert, dziugaite2015training} and 
(closely related) energy distances \cite{szekely2013energy, sejdinovic2013equivalence} are one possible class of metrics that could be used.
In this work, we focus on the use of computationally tractable optimal transport based metrics on the space of measures; in particular we use the sliced-Wasserstein metric~\cite{bonneel2015sliced}. Many other works have explored the use of these optimal transport (OT) based metrics for ML inference tasks~\cite{kolouri2019generalized, deshpande2019max, nadjahi2021sliced, nguyen2024energy,liutkus2019sliced}. There is also work on combining pushforward measures and Bayesian inference in~\cite{butler2018combining, butler2017consistent, marzouk2016introduction, el2012bayesian}. Solving inverse problems with Wasserstein loss is also a partially explored topic in~\cite{adler2017learning} and conditional flow matching~\cite{chemseddine2024conditional}. Some works have also looked at Wasserstein metrics between pushforward measures~\cite{sagiv2020wasserstein}. Wasserstein and sliced-Wasserstein metrics have found use in approximate Bayesian computation (ABC)~\cite{tavare1997inferring, beaumont2002approximate, nadjahi2020approximate} where the recovered measure have been shown to converge to the Bayesian posterior. Parameter estimation with Wasserstein metrics in a purer form is explored in~\cite{bernton2019parameter}.
In discussing OT based learning objective, it is important to mention entropy regularization, as in the Sinkhorn algorithm~\cite{cuturi2013sinkhorn}. In these works, entropy regularization is put on the finite-dimensional observational space. We will see that entropy regularization comes up differently in our proposed methodology.

\subsection{Notation}
\label{ssec:N}
Let $D \subset \mathbb{R}^d$ be bounded and open. We denote the boundary of the set $D$ by $\partial D$.
We use the $L^p(D)$ classes of $p-$th power Lebesgue integrable functions, $1 \le p < \infty$, extending
to $p=\infty$ in the usual way via the essential supremum. 
We denote by $C^\infty(D)$ the set of infinitely differentiable functions, and by $H^1_0(D)$ the Sobolev space of functions with one square-integrable weak derivative and homogeneous Dirichlet boundary conditions; we
denote by $H^{-1}(D)$ the dual of $H^1_0(D)$ with respect to the canonical pairing through Lebesgue integration over $D$.

Let $\cP(X)$ denote the space of probability measures on measurable space $X$.
Divergences on the space of probability measures are denoted by $\dd$, sometimes with subscript $i \in \{1,2\}.$  We denote by $f_\#\mu$ the pushforward measure given by $f_\#\mu(A) = \mu(f^{-1}(A))$ for all $\mu$ measurable sets $A$. We denote indexing over different instances of variables in a collection (such as a dataset or set of randomly sampled variables) with a superscript in parenthesis, accessing elements of a vector are done through a subscript in parenthesis, and incrementing (such as in summations) is done with a plain subscript. We denote by
$\langle \cdot, \cdot \rangle_A=\langle \cdot, A^{-1}\cdot \rangle$ the covariance weighted inner-product,
for any positive self-adjoint $A$, with induced norm $\|\cdot\|_{A}.$

Consider \eqref{eq:2} and assume that $z^{(n)} \sim \mu^\dagger \in \mathcal{P}(Z)$ where $Z \subset L^\infty(D)$ is separable. Let $F^\dagger: Z \to H^1_0(D)$ and let $g: H_0^1(D) \to \mathbb{R}^{d_y}$ denote a set of functionals. If we assume that $\varepsilon^{(n)}_\eta \sim \eta$ for $n = 1,\ldots,N$ are i.i.d noise variables then $y^{(n)} \sim \eta * (g \circ F^\dagger)_\# \mu^\dagger.$ The problem of interest is to recover from the observations $\{y^{(n)}\}_{n=1}^N$ the law $\mu^\dagger$ of the parameter field.

To be concrete we will work with inverse problems defined by the Darcy equation
\begin{subequations}
\label{eq:poisson_bc}  
\begin{align}
    \nabla\cdot(z\nabla u) + f &= 0, \quad \forall x \in  D,\\
    u &= 0, \quad \forall x \in \partial D.
\end{align}
\end{subequations}
Here $z$ denotes permeability and $u$ the piezometric head. The mapping $z \mapsto u$
may be viewed as mapping $F^\dagger: Z \to H^1_0(D),$ for appropriately defined $Z$. To be concrete
we will focus on this setting, and the problem of determining a prior on $z$ from noisy
linear functionals of $u$ defined by mapping $g: H^1_0(D) \to \bR^{d_y} \cong \bigl(H^{-1}(D)\bigr)^{d_y}.$ The reader will readily see that the ideas in the paper apply
more generally, and that consideration of the Darcy problem is simply for expository purposes.

\section{Residual-Based Neural Operator}
\label{sec:R}

This section is devoted to defining $\phi^*(\alpha)$ which appears in \ref{eq:ideal_loss3}.
Recall that $\phi$ are the parameters of the neural network surrogate PDE model, and that their optimization
depends on the underlying input measure on which the surrogate needs to be accurate; 
for this reason
they depend on $\alpha$, i.e.  the set of parameters characterizing the prior we want to learn.
To define function $\phi^*(\alpha)$ we proceed as follows. Let $F^\phi: Z \to H^1_0(D)$ be a parametric family of maps  approximating $F^\dagger$. Now define residual operator  $R: Z \times H_0^1(D) \to H^{-1}(D)$. Intuitively, we can express our PDE as $R(z, u) = 0.$
In the case of the Darcy equation \eqref{eq:poisson_bc} we may write
$$R(z, u) = \nabla \cdot (z \nabla u) + f,$$ and we have $ u \in H_0^1(D)$, $z\in Z \subset L^\infty(D)$, and $f\in L^2(D)$. Note next that we can write $u = F^\dagger(z)$ by the definition of forward map, hence we have
$R(z, F^\dagger(z)) = 0$ for all $z \in Z$. In order to incorporate this information into our loss functional, we define $R^\phi: Z \to H^{-1}(D)$ as
\begin{align}
R^\phi(z) = R(z, F^\phi(z)),
\end{align}
where $F^\phi$ is the parametric family of maps. In order to obtain a computable loss, we introduce a \textit{discretization operator} $\cO: H^{-1}(D) \to \Rb^{d_o}$, where $d_o$ is the dimension of the output of $\cO$. In particular, given a set of basis functions $\{v_i\}_{i=1}^{d_o}$ for $H_0^1(D)$, we can 
write\footnote{Here $\langle \cdot, \cdot \rangle$ denotes the duality pairing 
between $H^1_0(D)$ and $H^{-1}(D)$}
\begin{align}
\cO(R)_i = \langle v_i, R \rangle = \int_D v_i R(z, u)(x) \mathrm{d} x,
\end{align}
for $i = 1,\ldots,d_o$, for any given $v_i \in H_0^1(D)$, noting that integration by parts may
be used to to show well-definedness in the given function space setting. To compactly represent the discretization process with our emulator $F^\phi$, we define $\mathcal{O}^\phi: X \to \Rb^{d_o}$ as
\begin{align}
\cO^\phi(z) = \cO(R^\phi(z)).
\end{align}
We are now in a position to define our final loss functional using the constructions above. To learn the optimal parameter $\phi$ as $\alpha$ varies, we define the following coupled loss functional and associated
minimization problem:
 \begin{align} \label{eq:ideal_loss4} 
\sJ_4(\phi;\alpha) &= \dd_2\Bigl(\delta_0, (\cO^\phi \circ T^\alpha)_\#\mu_0)\Bigr)\tag{{Functional 4}},\\
 \phi^\star(\alpha) &= \underset{\phi}{\mathrm{arg min}}\;\sJ_4(\phi;\alpha); \nonumber
\end{align}
here  $\dd_2$ is a divergence term between probability measures on $\Rb^{d_o}$ and $\delta_0$ is the Dirac measure at zero. Minimizing $\sJ_4(\cdot;\alpha)$ for given $\alpha$ determines a residual-based approximation
of $F^\dagger$, accurate with respect to $(T^\alpha)_\#\mu_0$. 
Using this expression shows that \ref{eq:ideal_loss3} and \ref{eq:ideal_loss4} define
a bilevel optimization scheme~\cite{sinha2017review, holler2018bilevel}. This scheme is the heart of our proposed methodology. Using
the bilevel approach avoids balancing the contributions of \ref{eq:ideal_loss3} and \ref{eq:ideal_loss4} that arise from an additive approach.
A similar bilevel optimization scheme is employed in~\cite{zhang2024bilo}, for similar reasons,
to solve a different problem. 
In the next section we provide specific examples $\dd_1$ and $\dd_2$ and discuss some of their properties.

\section{Choice of Divergences} \label{sec:M}
In order to instantiate the bilevel optimization scheme defined by 
\ref{eq:ideal_loss3}, \ref{eq:ideal_loss4} to obtain an implementable algorithm, we now define 
specific choices of the divergence terms $\dd_1$ and $\dd_2$ (Subsection \ref{ssec:M}) and discuss empirical approximation of the input measures (Subsection \ref{ssec:E}) required to evaluate these
divergences in practice. And, to further establish a context for our work
on learning priors, we make a connection  to Bayesian inversion in the setting of equation \eqref{eq:1} when $N=1$ (Subsection \ref{ssec:B}.)

\subsection{Divergences}
\label{ssec:M}
In this work, to effectively and efficiently compare empirical measures, we will use the sliced-Wasserstein (SW) distance to define the divergence term for $\dd_1$ and Wasserstein distance for $\dd_2$. To define precisely what we do, we first introduce the weighted
Wasserstein distance
\begin{align}\label{eq:WassB}
    \mathsf{W}_{2,\B}^2(\nu, \mu) = \inf_{\gamma \in \Pi(\nu, \mu)} \int_{\Rb^d \times \Rb^d} \|x - y\|_{\B}^2 \md \gamma(x, y), 
\end{align}
where coupling $\Pi(\nu, \mu)$ is the set of all joint probability measures on $\Rb^d \times \Rb^d$ with marginals $\nu$ and $\mu$, and $\B$ is positive and self-adjoint. We let $\mathsf{W}_{2}:=\mathsf{W}_{2,\mathrm{I}}.$
The following lemma shows that the weighted squared Wasserstein-2 distance can be seen as the squared Wasserstein-2 distance of pushforwards of the original measures. 
\begin{lemma}\label{lemma:pushW}
For $P_\B(\cdot) = \B^{-1/2}\,\cdot$ it follows that
\begin{align}
    \mathsf{W}^2_{2, \B}(\nu, \mu) =  \mathsf{W}^2_{2}(P_{\B\#}\nu, P_{\B\#}\mu). 
\end{align}
\end{lemma}

\begin{proof}
    See Appendix~\ref{app:sec:pushW}.
\end{proof}

We now define the weighted and sliced-Wasserstein distance by
\begin{align}
    \mathsf{SW}^2_{2, \B}(\nu, \mu) 
    = \int_{\mathbb{S}^{d-1}} \mathsf{W}_{2}^2(P^\theta_{\B\#}\nu, P^\theta_{\B\#}\mu) \md \theta,
\end{align}
where  $P^\theta_\B(\cdot)=\langle \B^{-\frac12}\cdot, \theta \rangle.$
The sliced-Wasserstein distance leads to a computationally efficient alternative to the Wasserstein distance because it may be implemented by
Monte Carlo approximation of integration over $\theta$ and then each slice, resulting from
a randomly chosen $\theta$, involves only evaluation of a Wasserstein distance between probabilities
in $\bR.$

Assuming that $\eta = \mathcal{N}(0, \Gamma)$ we then consider the losses
\ref{eq:ideal_loss3}, \ref{eq:ideal_loss4} with $\dd_1(\cdot,\cdot)=\frac{d_y}{2}\mathsf{SW}^2_{2, \Gamma}(\cdot,\cdot)$ and $\dd_2(\cdot,\cdot)=\mathsf{W}^2_{2}(\cdot,\cdot),$ so that
\begin{subequations}
\label{eq:need}
\begin{align}
    \label{eq:J3_swass}
    \sJ_3(\alpha) &= \frac{d_y}{2}\mathsf{SW}^2_{2, \Gamma}\Bigl(\nu, 
    \eta * (g \circ F^{\phi^*(\alpha)} \circ T^\alpha)_{\,\#} \mu_0\Bigr) + h(\alpha),\\
    \label{eq:J4_wass}
    \sJ_4(\phi;\alpha) &=  \mathsf{W}_{2}^2\Bigl(\delta_0, (\cO^\phi \circ T^\alpha)_{\,\#}\mu_0\Bigr),\\
     \phi^\star(\alpha) &= \underset{\phi}{\mathrm{arg min}}\;\sJ_4(\phi;\alpha).
\end{align}
\end{subequations}
\begin{remark}
    We highlight that using the squared Wasserstein-2 metric between a Dirac at zero and an abitrary measure reduces to computing an expected squared 2-norm of the samples drawn from that measure (Lemma \ref{lemma:weightedW}). Hence, \eqref{eq:J4_wass} reduces to computing 
    $$\sJ_4(\phi;\alpha) = \Eb_{z\sim  (T^\alpha)_{\,\#}\mu_0 } \|\cO^\phi(z)\|^2_2.$$
    Framing \eqref{eq:J4_wass} this way gives a different interpretation to the commonly used physics-informed ML loss function. \diabox0{cblack}
\end{remark}

\subsection{Empiricalization}
\label{ssec:E}

The optimization problem in \eqref{eq:need} forms the basis of our computational methodology in
this paper. However, to implement it, we need to use empirical approximations of the two input
measures that define the loss function $\sJ_3.$
There are two ways in which evaluation of $\sJ_3$ defined by \eqref{eq:need} must be empiricalized
to make a tractable algorithm:
\begin{itemize}
    \item measure $\nu$ is replaced by measure $$\nu^N=\frac{1}{N} \sum_{n=1}^N  \delta_{y^{(n)}},$$
    reflecting the fact that working with $\nu$ is not computationally tractable, but samples from $\nu$ are
    available;
    \item measures $\nu^N$ and $(g \circ F^{\phi^*(\alpha)} \circ T^\alpha)_{\,\#} \mu_0$ are replaced by their empirical approximations using, in each case, $N_s$ independent samples,
    reflecting the fact that working with $(g \circ F^{\phi^*(\alpha)} \circ T^\alpha)_{\,\#} \mu_0$ is not computationally tractable, together with the computational simplicity arising from
    using the same number of samples in each argument of the divergence;
    \item measure $(\cO^\phi \circ T^\alpha)_{\,\#}\mu_0$ is empiricalized using $N_r$ independent samples
    from $\mu_0$.
\end{itemize}
\begin{remark}\label{rem:scalingVBayes}
            In contrast to empirical/hierarchical Bayesian approaches for similar problems, the proposed methodology side-steps the sampling of a challenging posterior distribution with dimension that grows like  $O(NM)$, where $N$ is the number of physical systems from which we have data, and $M$ is the dimensionality of the parameters, $z^{(n)}$, of the individual physical systems. Instead, this is replaced with the straightforward empiricalization of a pushforward measure with $N_s$ samples of dimension $M$. \diabox0{cblack}
\end{remark}
In the numerical examples explored in Section~\ref{sec:N}, for 1D Darcy $M=20$, for 2D Darcy $M=400$, these will be the number of bases parametrizing the permiability field in the two setups respectively.

\subsection{Connection to Bayes Theorem}
\label{ssec:B}
Consider the inverse problem defined by \eqref{eq:1} in setting $N=1:$
\begin{align}
\label{eq:11}
y = \cG(z) + \varepsilon,
\end{align}
where $\varepsilon \sim \eta := \mathcal{N}(0, \Gamma)$ and $\cG: Z \to \bR^d.$
Consider \ref{eq:ideal_loss} with $\nu=\delta_{y}$, $\cG(\cdot)$ replacing 
$(g \circ F^\dagger)(\cdot)$\footnote{There is nothing intrinsic to the factorization  $\cG=g \circ F^\dagger$ in this subsection; results are expressed purely in terms of $\cG$.} and choice of $\dd_1$ as in Subsection \ref{ssec:M}:
\begin{equation}
\label{eq:ideal_loss4_N1}    
{\sJ}(\mu) = \frac{d}{2}\mathsf{SW}^2_{2, \Gamma}\Bigl(\delta_y, \eta * (\cG)_\# \mu \Bigr) + {\cH}(\mu).
\end{equation}

Thus we have returned to the population loss description of the problem in the setting where
optimization to determine the prior is over all probability measures, not the parameterized
family that we will use for computation. Thus the regularization is on the space of probability
measures. The message of the following theorem is that, with this problem formulation on the
space of measures, in the case $N=1$ and with appropriate
choice of regularization, Bayes theorem is recovered.

\begin{theorem} \label{t:T}
    Define $\cH(\mu) := D_\KL(\mu||\mu^{\mathrm{prior}})$ for some probability measure 
    $\mu^{\mathrm{prior}}$ $\in$ $\cP(Z).$
    Consider the Bayesian inverse problem defined by \eqref{eq:11} with $z \sim \mu^{\mathrm{prior}}$ independent
    of $\varepsilon \sim \eta := \mathcal{N}(0, \Gamma).$  Then the minimizer of \eqref{eq:ideal_loss4_N1}  over the set of probability measures 
    \{$\mu \in \cP(Z)$ : $\bE_{y'\sim \eta * (\cG_\#) \mu} \|y'\|_\Gamma^2 < \infty\}$ is the Bayesian posterior given by
    \begin{subequations}
     \label{th:bays_equiv}   
    \begin{align}
         \mu^y(A) &= \frac{1}{\cZ}\int_A \exp\left(-\frac{1}{2}\|y - \cG(z)\|^2_\Gamma\right){\md \mu^{\mathrm{prior}}(z)},\\
        \cZ &= \int_Z \exp\left(-\frac{1}{2}\|y - \cG(z)\|^2_\Gamma\right){\md \mu^{\mathrm{prior}}(z)}.
    \end{align}
    \end{subequations}
\end{theorem}

\begin{remark} As discussed in the introduction to this paper, our method learns a \textit{data-informed probability measure} -- a generative model -- for $z$, which may be used as a \textit{prior} for
downstream inference tasks. The theorem shows that, in the special case of $N=1$, the learned prior actually coincides with the Bayesian posterior for the inverse problem defined by \eqref{eq:1}, if regularizer $\cH$ is chosen appropriately. This is entirely consistent with our broader agenda when $N>1$ as the posterior distribution is the natural prior for downstream tasks in Bayesian inference. \diabox0{cblack}
\end{remark}

To prove this theorem we establish a sequence of lemmas. The first shows
that the weighted Wasserstein-2 distance may be simplified when one of its argument is a Dirac:
\begin{lemma}\label{lemma:weightedW}
For any $y \in \Rb^{d}$ and $\mu \in \mathcal{P}(\Rb^d)$
\begin{align}
\label{eq:A}
 \mathsf{W}^2_{2,\Gamma}(\delta_y, \mu) = \Eb_{x \sim \mu} \|y-x\|_{\Gamma}^2. 
\end{align}
\end{lemma}
\begin{proof}
    See Appendix~\ref{app:sec:weightedW}.
\end{proof}
Using the definition of pushforward, it follows from Lemma \ref{lemma:weightedW} that
\begin{align}
\label{eq:B}
        \mathsf{W}_{2, \Gamma}^2(\delta_y, (\cG)_\#\mu) =
        \Eb_{z \sim \mu} \|y-\cG(z)\|_{\Gamma}^2=\int_{\Rb^{d}} \|y - \cG(z)\|_\Gamma^2\md\mu(z).
    \end{align}
The following lemma shows
that the sliced-Wasserstein metric also simplifies when one of its argument is a Dirac:
\begin{lemma}\label{lemma:w_sw_equiv} Let $y\in\Rb^d$ and $\mu \in \mathcal{P}(\Rb^d)$, and assume $\bE_{z\sim\mu}\|z\|_\Gamma^2 < \infty$. Then
\begin{align}
    \mathsf{SW}_{2,\Gamma}^2(\delta_{y}, \mu)=\frac{1}{d}  \mathsf{W}_{2,\Gamma}^2(\delta_{y}, \mu).
\end{align}
\end{lemma}
\begin{proof}
  See Appendix~\ref{app:w_sw_equiv}
\end{proof}
The final lemma concerns convolution:
\begin{lemma}\label{lemma:gradient_noise_conv}
Let $\sI: \cP(Z) \to \bR$ be defined by $\sI(\cdot) = \mathsf{W}_{2,\Gamma}^2(\delta_y, \cdot)$.
Let $\eta$ be a centred random variable and finite second moment. Then
    \begin{align}
        \sI(\eta*\mu) = \sI(\mu) + \bE_{\varepsilon \sim \eta} \|\varepsilon\|_{\Gamma}^2.\nonumber
    \end{align} 
\end{lemma}
\begin{proof}
    See Appendix~\ref{proof:gradient_noise_conv}.
\end{proof}

\begin{proof}[Proof of Theorem \ref{t:T}]
By equation \eqref{eq:B} and Lemmas \ref{lemma:w_sw_equiv} and \ref{lemma:gradient_noise_conv}
it follows that the loss function \eqref{eq:ideal_loss4_N1} can be written as
\begin{align}\label{eqn:functional_one_data}
    \sJ(\mu) = \frac{1}{2}\Eb_{z\sim \mu}\| y - \cG(z) \|^2_\Gamma + \cH(\mu) + C_1,
\end{align}
where $C_1$ is independent of $\mu.$ With the choice $\cH(\mu) := D_\KL(\mu||\mu^{\mathrm{prior}})$ we
see that $\sJ(\mu)=D_\KL(\mu||\mu^{y})+C_2$, where $C_2$ is independent of $\mu$, and is determined by $C_1$
and $\cZ.$ The result follows
since $D_\KL(\cdot||\mu^{y})$ is minimized at $\mu^y.$
\end{proof}
\begin{remark} In order to use Lemmas~\ref{lemma:w_sw_equiv} and   \ref{lemma:gradient_noise_conv} in the proof of Theorem~\ref{t:T}, we assume the pushforward measure $\eta * (\cG_\# \mu)$ has $\|\cdot\|_\Gamma^2$-moment finite as stated in Theorem~\ref{t:T}. This will follow from second moments of $\cG_\# \mu.$
Such results can often be established using the Fernique theorem; see
for example \cite{charrier2012strong,dashti2011uncertainty}. \diabox0{cblack}
\end{remark}

\begin{remark}
\label{rem:2}
Now consider functional $\sJ(\cdot)$ given by \eqref{eq:ideal_loss4_N1}, replacing 
$\cG$ by the specific choice $g \circ F^\dagger$ arising from the inverse problem defined
by $\eqref{eq:2}$ in the case where $N=1.$ If we now parameterize target measure $\mu$ by
$\mu=T^\alpha_\#\mu_0$, make the choice of KL divergence for $\cH$ and view $\sJ(\cdot)$ 
as parameterized by $\alpha \in \bR^{d_\alpha}$ rather than by $\mu \in \bP(Z),$ we obtain
\begin{align*}  
{\sJ}(\alpha) = \frac{d}{2}\mathsf{SW}^2_{2, \Gamma}\Bigl(\delta_y, \eta * (g \circ F^\dagger \circ T^\alpha)_\# \mu_0 \Bigr) + D_\KL((T^\alpha)_{\#}\mu_0||\mu^{\mathrm{prior}}).
\end{align*}
By the same reasoning used in the proof of Theorem \ref{t:T} we may rewrite this as
\begin{align*}
    {\sJ}(\alpha) = \frac{1}{2}\Eb_{z\sim (T^\alpha)_{\#}\mu_0}\| y - (g \circ F^\dagger)(z) \|^2_\Gamma + D_\KL((T^\alpha)_{\#}\mu_0||\mu^{\mathrm{prior}}).
\end{align*}
From this form it is clear that minimizing ${\sJ}(\alpha)$ over $\alpha \in \bR^{d_\alpha}$ is simply
the variational Bayes methodology applied to approximate the Bayesian posterior $\mu^y$ given by 
equation \eqref{th:bays_equiv}. Furthermore this provides motivation for the consideration of \ref{eq:ideal_loss2} to determine the prior on $z$, in the case where $N>1$, noting that it reduces to variational Bayes when $N=1$ with appropriate choices of $\dd_1(\cdot,\cdot)$ and $h(\cdot).$  \diabox0{cblack}
\end{remark}

\section{Algorithms} 
\label{sec:A}

\begin{algorithm}[t]
\caption{Prior Calibration}
\label{alg:prior_cal}
\begin{algorithmic}[1]
\STATE{Initialize $\alpha_0$, $T$ (\text{number of iterations}), $N_s$ (\text{number of samples for $\sJ_2$}), $F^\dagger$ (\text{forward model})}
\FOR{$t=1, \hdots, T$}
\FOR{$i=1, \hdots, N_s$}
\STATE{Sample $\varepsilon_0^{(i)}\sim \mu_0$}
\STATE{Sample $z^{(i)}\sim T^{\alpha_{t-1}}(\varepsilon_0^{(i)})$ }
\STATE{Sample $\varepsilon_\eta^{(i)}\sim\eta$}
\STATE{Compute ${y'}^{(i)} = g \circ F^\dagger(z^{(i)}) + \varepsilon_\eta^{(i)}$}
\STATE{Sample $y^{(i)}\sim\nu^N$}
\ENDFOR
\STATE{Compute $\sJ_2(\alpha_{t-1})$ from the $N_s$ samples $\{y^{(i)}, {y'}^{(i)}\}_{i=1}^{N_s}$.}
\STATE{$\alpha_t = \text{OPTIMISER}(\alpha_{t-1}, \sJ_2(\alpha_{t-1}))$}
\ENDFOR 
\RETURN $\alpha^\star \leftarrow \alpha_T$.
\end{algorithmic}
\end{algorithm}

In this section we discuss practical aspects pertaining to the implementation of the proposed methodologies.
Firstly, although Remark \ref{rem:2} suggest a specific choice for regularization $h(\alpha)$ in
\ref{eq:ideal_loss2} or \ref{eq:ideal_loss3}; in practice it is not computationally straightforward to work with this choice
and simpler choices are made. Secondly, the evaluation of the objective functions is carried out
using the empiricalization described in Subsection \ref{ssec:E}. Thirdly the  physics-based residual
can be evaluated using Lemma \ref{lemma:weightedW} and the considerations of Subsection \ref{ssec:E}. Fourthly, once loss functions are evaluated we obtain gradients with respect to prior (and operator) parameters through back-propagation with standard ML library tools, such as \text{JAX}~\cite{jax2018github}.

In Algorithm~\ref{alg:prior_cal} we show how to implement the proposed inference methodology of~\ref{eq:ideal_loss2} for the task of prior calibration given a forward model $F^\dagger$.
In Algorithm~\ref{alg:prior_cal_op_learn} we show how to implement \ref{eq:ideal_loss3} for joint prior calibration and operator learning. We note that to correctly implement the bilevel optimization scheme proposed to minimize \ref{eq:ideal_loss3} we must differentiate through the lower-level optimization steps given by \ref{eq:ideal_loss4} to take into account the dependence of $\phi^\star(\alpha)$ on $\alpha$. This incurs additional computational costs; efficient approximation methods for this task will be explored in future works.

We note, only the outer-most parameter update loops must be computed sequentially. All other loops can be efficiently computed in parallel in a GPU efficient manner. In all experiments we make use of the Adam optimizer~\cite{kingma2014adam}. The sliced-Wasserstein implementation is based on that of~\cite{flamary2021pot}. If one is interested in applying this methodology to nonlinear PDEs, computing~\ref{eq:ideal_loss2} requires the solution of $N_s$ nonlinear systems of equations at every parameter update step. However, this is not the case for the method using~\ref{eq:ideal_loss3}, \ref{eq:ideal_loss4} as we only need to compute $N_r$ residuals, which is done the same manner for linear or nonlinear PDEs. (See, also, Subsection \ref{ssec:E} for definitions of $N_s, N_r.$)

\begin{algorithm}[t]
\caption{Prior Calibration and Operator Learning}
\label{alg:prior_cal_op_learn}
\begin{algorithmic}[1]
\STATE{Initialize $\alpha_0, \phi_{0}$, $T$ (\text{number of outer-loop $\alpha$ updates}), $L$ (number of inner-loop $\phi$ updates), $N_s$ (\text{number of samples for $\sJ_3$}), $N_r$ (number of samples for $\sJ_4$)}
\FOR{$t=1, \hdots, T$}
\FOR{$l=1, \hdots, L$}
\FOR{$j=1, \hdots, N_r$}
\STATE{Sample $\varepsilon_0^{(j)}\sim \mu_0$}
\STATE{Sample $z^{(j)}\sim T^{\alpha_{t-1}}(\varepsilon_0^{(i)})$ }
\STATE{Compute $r^{(j)} = \cO^{\phi_{l-1}(\alpha_{t-1})}(z^{(j)})$}
\ENDFOR
\STATE{Compute $\sJ_4(\alpha_{t-1}, \phi_{l-1}(\alpha_{t-1}))$ from the $N_r$ samples $\{r^{(j)}\}_{j=1}^{N_r}$.}
\STATE{$\phi_{l}(\alpha_{t-1}) = \text{OPTIMISER}(\phi_{ l-1}(\alpha_{t-1}), \sJ_4(\alpha_{t-1}, \phi_{l-1}(\alpha_{t-1}))$}
\ENDFOR
\STATE $\phi^\star(\alpha_{t-1})\leftarrow\phi_L(\alpha_{t-1})$
\FOR{$i=1, \hdots, N_s$}
\STATE{Sample $\varepsilon_0^{(i)}\sim \mu_0$}
\STATE{Sample $z^{(i)}\sim T^{\alpha_{t-1}}(\varepsilon_0^{(i)})$ }
\STATE{Sample $\varepsilon_\eta^{(i)}\sim\eta$}
\STATE{Compute ${y'}^{(i)} = g \circ F^{\phi^\star(\alpha_{t-1})}(z^{(i)}) + \varepsilon_\eta^{(i)}$}
\STATE{Sample $y^{(i)}\sim\nu^N$}
\ENDFOR
\STATE{Compute $\sJ_3(\alpha_{t-1}, \phi^\star(\alpha_{t-1}))$ from the $N_s$ samples $\{y^{(i)}, {y'}^{(i)}\}_{i=1}^{N_s}$.}
\STATE{$\alpha_t = \text{OPTIMISER}(\alpha_{t-1}, \sJ_3(\alpha_{t-1}, \phi^\star(\alpha_{t-1})))$}
\ENDFOR 
\RETURN $\{\alpha^\star \leftarrow \alpha_T, \phi^\star(\alpha^\star)\leftarrow \phi^\star(\alpha_T)\}$.
\end{algorithmic}
\end{algorithm}

\section{Numerical Results for Darcy Flow}
\label{sec:N}
We now illustrate the performance of our methodology to learn generative models for priors,
based on indirect observations. All our numerical examples are in the setting of equation \eqref{eq:poisson_bc}.  The presented PDE is to be interpreted in the weak sense; see Appendix \ref{app:weak_form_residuals} for a short discussion on the weak form and  approximate computational
methods exploiting GPU-efficient array shifting operations. We consider two classes of 
coefficient function $z:$ the first is a class of piecewise constant functions with discontinuity
sets defined as level sets of a smooth field, in Subsection \ref{ssec:level_sets}; the second
is a class of functions defined as the pointwise exponential of a smooth field, in Subsection \ref{ssec:logn}.
In both cases we will write the prior as pushforward of a Gaussian measure, on the underlying
smooth field, under a parameter-dependent map
$T^\alpha$; we refer to the first class as \emph{level set priors} and the second as
\emph{lognormal priors}. We attempt to learn the parameters $\alpha$ through minimization of either
\ref{eq:ideal_loss2} (using numerical PDE solver) or \ref{eq:ideal_loss3} (using residual-based
operator learning); in all cases we use the framework of Section \ref{sec:M}.
We denote the true measure which we wish to recover by $\mu^\dagger.$

For simple problems regularization may not be needed. For more challenging problems, regularization on $\alpha$ may help with numerical stability, particulariliy when the surrogate model $F^{\phi^\star(\alpha)}$ is concurrently learned. Hence, for the 1D Darcy problems we ommit the regularizer, and for the 2D Darcy examples we use regularization of the form $h(\alpha)=1/({2\sigma_h^2})\|\log(\alpha)-m_h\|^2_2$ in which the log
is applied component-wise; we refer to $\sigma_h$ and $m_h$ as the standard-deviation and mean of the regularizer term. 
We have not found it necessary to regularize the neural operator parameters as we are not interested in recovering a specific value for $\phi^\star(\alpha)$, rather any value which achieves a small error when evaluating the PDE residual of the output is sought.

For all experiments, we use 1000 slicing directions, $\theta$, to evaluate the sliced-Wasserstein term, and we minimize the relevant loss functions using Adam~\cite{kingma2014adam} with a learning rate of $10^{-2}$ decayed by half four times on the parameters $\alpha$, other than where explicitly stated. 
All experiments are run on a 24GB NVIDIA RTX 4090 GPU.

The numerical results we present substantiate the following conclusions:
\begin{itemize}
    \item The proposed methodology using~\ref{eq:ideal_loss2} can effectively recover the true parameters of a level set prior and a lognormal prior.
    \item The proposed methodology using~\ref{eq:ideal_loss2} can effectively recover the true parameters of a level set prior despite using a smoothening of the level set formulation, so that there is prior
    mis-specification.
    \item We highlight the impact of dataset size ($N$), and the number of samples used to estimate the loss ($N_s$), on the quality of the prior parameter estimation; this demonstrates the strength of distributional inference for prior calibration.
    \item The proposed methodology using~\ref{eq:ideal_loss3},~\ref{eq:ideal_loss4} can achieve comparable parameter estimation accuracy to that under~\ref{eq:ideal_loss2}; thus jointly estimating the operator 
    approximation and the parameters of the prior is both feasible and, for reasons of efficiency, will be desirable in settings where $F^\dagger$ is expensive to
    evaluate, is not differentiable or is not available.
    \item We show how, under certain circumstances, prior parameters may be unidentifiable, and we show limitations of the proposed methodology in this setting.
\end{itemize}

\subsection{Level Set Priors}\label{ssec:level_sets}

\begin{figure}[t]\label{fig:spectrum_decay_GRF}
    \centering
     \subfigure[Normalized Spectrum, $\beta=4$]{
     \includegraphics[width=0.35\linewidth]{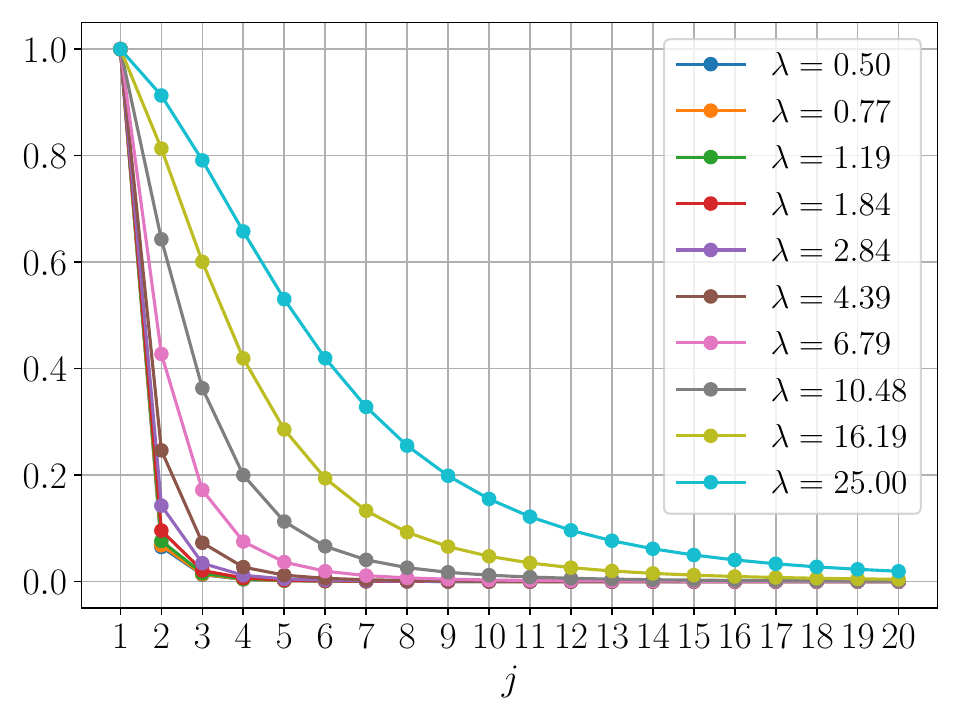}}
    \subfigure[Smoothed level set function]{
    \includegraphics[width=0.35\linewidth]{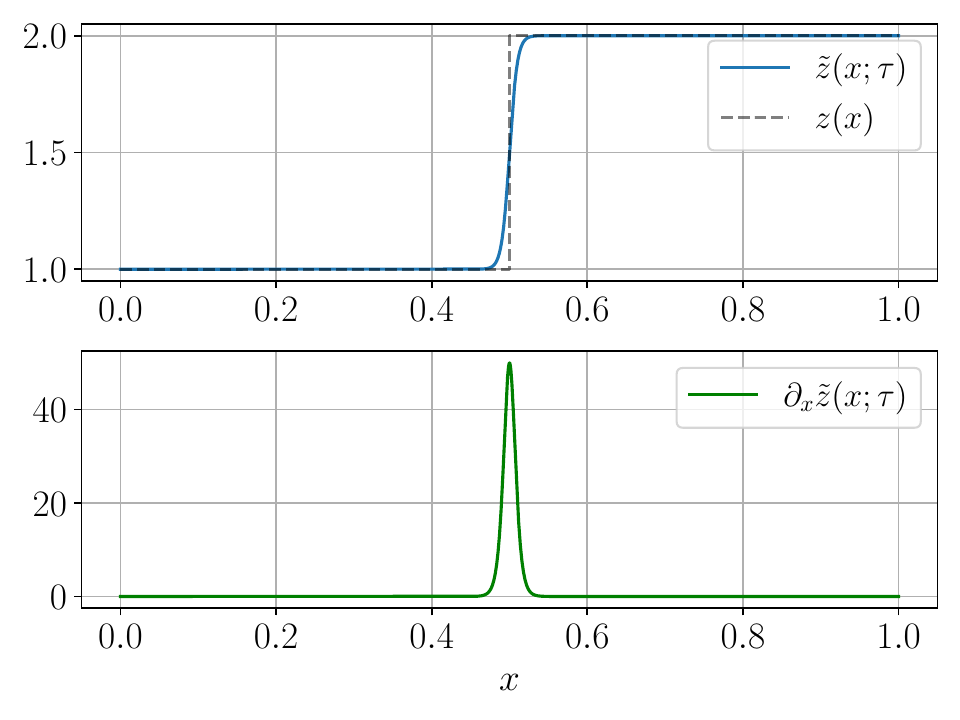}}
    \caption{(a) Spectrum decay of the square root of the eigen values for the covariance operator~\eqref{eqn:GRF_operator} (which corresponds to the standard deviation of the basis expansion coefficient of each mode $\varphi_{j}(x)$) for $\beta=4$ and different choices of lengthscale parameter $\lambda$. As $\lambda$ increases, more modes play a significant role. (b) A comparison of smooth level set function $\tilde{z}$ and sharp level set function $z$ transitioning at $x=0.5$ for $\tau=100$.}
\end{figure}

Prior construction here is based on the methodology introduced in~\cite{dunlop2017hierarchical, iglesias2014well}. We define $\alpha=\{\kappa^-,\kappa^+,\lambda\}$ where
 $\kappa^{\pm} \in (0,\infty)$ and $\lambda>0.$ 
We now define a function $z: C^\infty(D) \times \Rb^2 \to L^\infty(D)$ 
which will take a $\lambda-$dependent function in $C^\infty(D)$, and
the two values $\kappa^{\pm} \in \Rb^2$ to create a function in $L^\infty(D).$
We introduce parameterized family of measures
$(T^\alpha)_{\#}\mu_0$ on $L^\infty(D)$ by pushing forward a Gaussian measure on 
$a \in C^\infty(D)$ under map $T^\alpha$ to define a measure on $z.$ To this end we define $z$ by
\begin{align}
    z(x) &=  \kappa^-\mathbbm{1}_{D^-(a)}(x)+\kappa^+\mathbbm{1}_{D^+(a)}(x),\\
    D^-(a) &= \{x\in D| a(x)<0\}, \quad D^+(a) = \{x\in D| a(x) \ge 0\}.
\end{align}
To complete the description of $T^\alpha$ we construct $a$ as a $\lambda-$dependent Gaussian random field.
To do this we fix a scalar $\beta>0.$ The construction of $a$ differs in details between dimensions one and two. For a 1D physical domain $D$, we define the $\lambda-$dependent
Gaussian measure on $a$ through the Karhunen-Lo\`eve expansion
\begin{align*}
    a(x;\lambda,\beta) &= \sum_{j=1}^{J} \Bigl(j^2\pi^2+\lambda^2\Bigr)^{-\beta/2} \varepsilon_j \varphi_{j}(x), \quad \varepsilon_j \sim \NPDF(0, 1)\; \mathrm{i.i.d,}
\end{align*}
where $\varphi_{j}(x) = \cos(j\pi x).$
In dimension two we generalize to obtain
\begin{align*}
    a(x;\lambda,\beta) &= \sum_{j=1, k=1}^{J, K} \Bigl((j^2 + k^2)\pi^2+\lambda^2\Bigr)^{-\beta/2} \varepsilon_{j,k} \varphi_{j,k}(x), \quad \varepsilon_{j,k} \sim \NPDF(0, 1)\; \mathrm{i.i.d},
\end{align*}
with $\varphi_{jk}(x) = \cos(j\pi x_{(1)}) \cos(k\pi x_{(2)}).$
For all experiments, we fix $\beta = 4$. 
\begin{figure}[t]\label{fig:Data_1D_PriorCal}
    \centering
    \subfigure[several $u$ and $y$]{
    \includegraphics[width=0.24\linewidth]{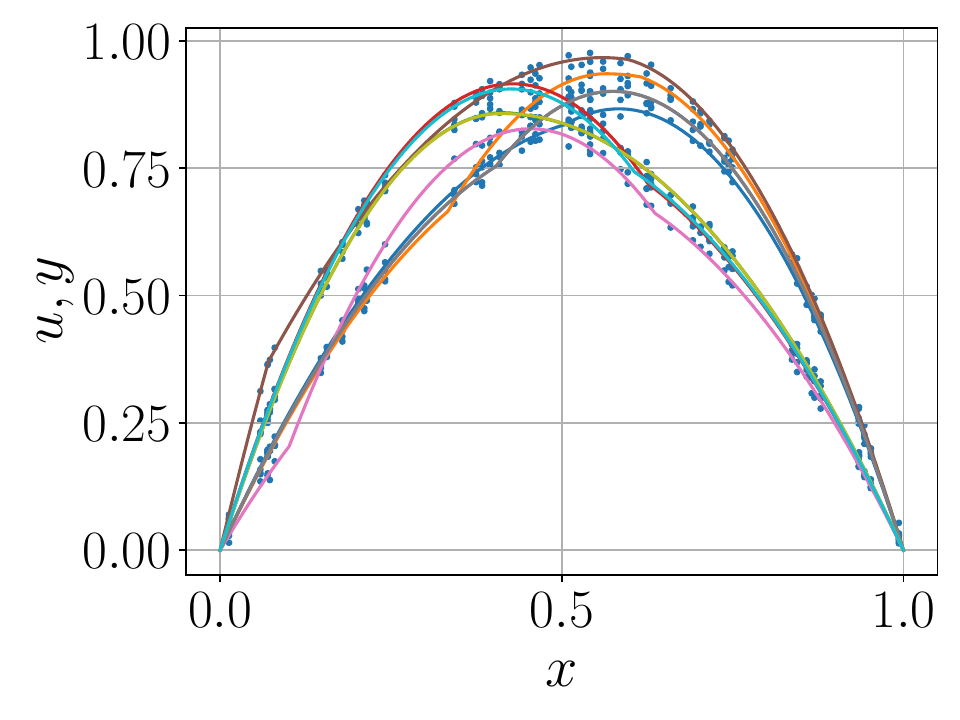}}
    \hspace{-.5em}
    \subfigure[Random fields $\bar{a}$]{\includegraphics[width=0.24\linewidth]{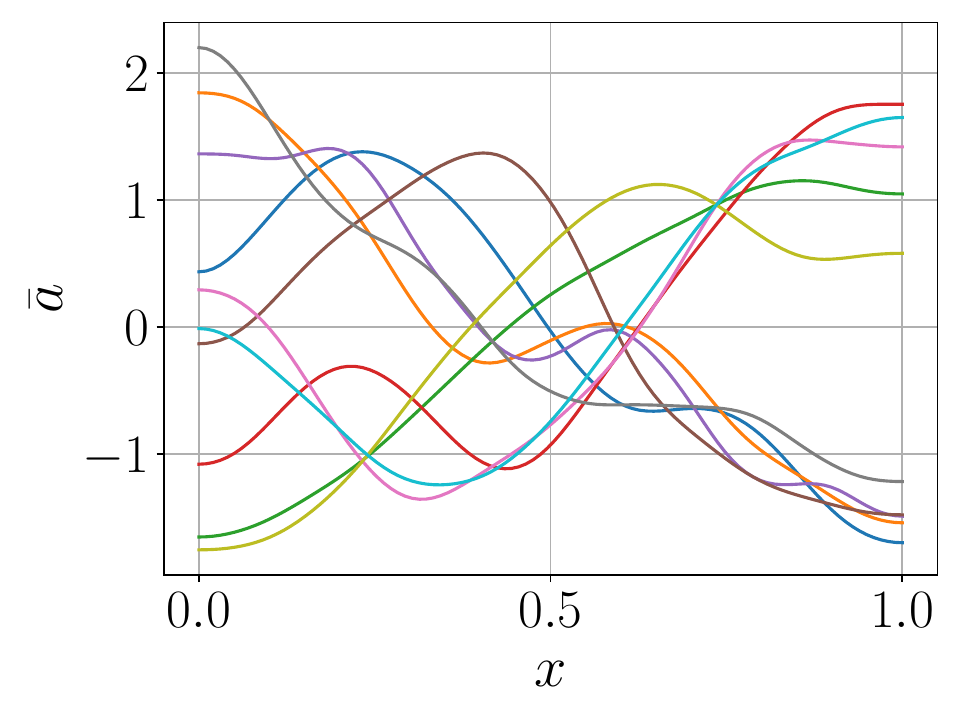}}
    \hspace{-.5em}
    \subfigure[Sharp $z$]{\includegraphics[width=0.24\linewidth]{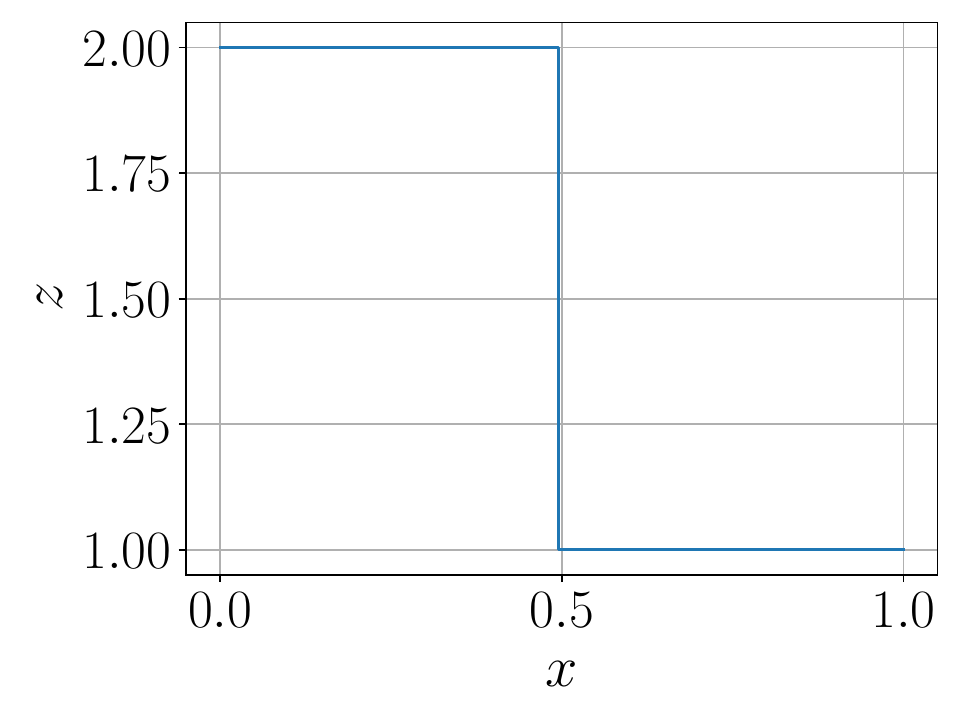}}
    \hspace{-.5em}
    \subfigure[Smooth $z$, $\tau=10$]{\includegraphics[width=0.24\linewidth]{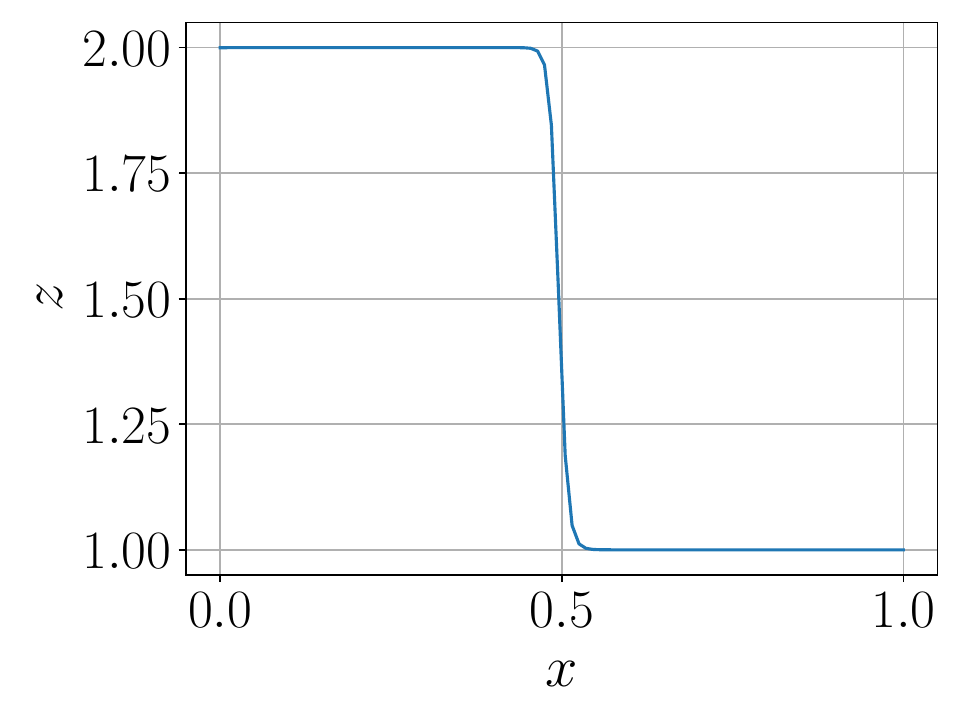}}
    \caption{(a)-(c) Data Generated from sharp prior $\mu^\dagger$, which are the PDE solution fields with the observational data at points, the normalized GRF $\bar{a}$ samples, and one example of a level set function resulting from one of the $\bar{a}$ fields. (d) The same diffusivity field $\tilde{z}$  as in (c), but smoothed with~\eqref{eqn:smooth_levelset_function} and $\tau=10$.}
\end{figure}

Note that, for $J=\infty$ in dimension one (resp. 
$(J,K)=\infty$ in dimension two)
$a \sim \NPDF(0,C_{\lambda,\beta})$ where
\begin{align}\label{eqn:GRF_operator}
    C^{-1}_{\lambda,\beta}=\bigl(-\Delta+\lambda^2 \I\bigr)^\beta,
\end{align}
and $-\Delta$ is the Laplacian equipped with homogeneous Neumann
boundary conditions. Draws from this measure have H\"older regularity up to
exponent $\beta-d/2$ where $d$ is dimension of domain $d.$
This statement about the probability measure from which $a$ is drawn is approximate when $J$ is (resp. $(J,K)$ are) finite.  Because of the countable nature of the construction of the level set prior, it lies in
a separable subspace $Z$ of $L^\infty(D).$ 
The derivative of the objective functional with respect to $\lambda$ is not well behaved for the \textit{sharp} level set setup as $\partial_{a}{z}$ is zero almost everywhere.
Hence we introduce a \textit{smoothened} level set parametrization of $z(\alpha)$ as 
\begin{align}\label{eqn:smooth_levelset_function}
    \tilde{z}(x; \tau) &= \frac{1}{2}\tanh(\tau\,\bar{a})(\kappa^+ - \kappa^-) + \frac{1}{2}(\kappa^+ - \kappa^-) + \kappa^-,\\
    \bar{a} &= a/{\|a\|_{L^2(D)}},
\end{align}
where the parameter $\tau$ controls the sharpness of transition from $\kappa^-$ to $\kappa^+$. 
In Fig.~\ref{fig:spectrum_decay_GRF}(a) we show the spectrum decay of the GRF generated from the covariance operator \eqref{eqn:GRF_operator} in dimension one.
Figure~\ref{fig:spectrum_decay_GRF}(b) shows the smoothing of the sharp level set function $z$ and its spatial derivative. In Fig.~\ref{fig:Data_1D_PriorCal} we show 10 sampled random fields $\bar{a}$ and the associated PDE solutions for the 1D Darcy problem. We only show one diffusion field $z$ for clarity.
\begin{figure}[t]\label{fig:convergence_1D_PriorCal}
    \centering
    \subfigure[Convergence $\mu^\dagger$]{\includegraphics[width=0.48\linewidth]{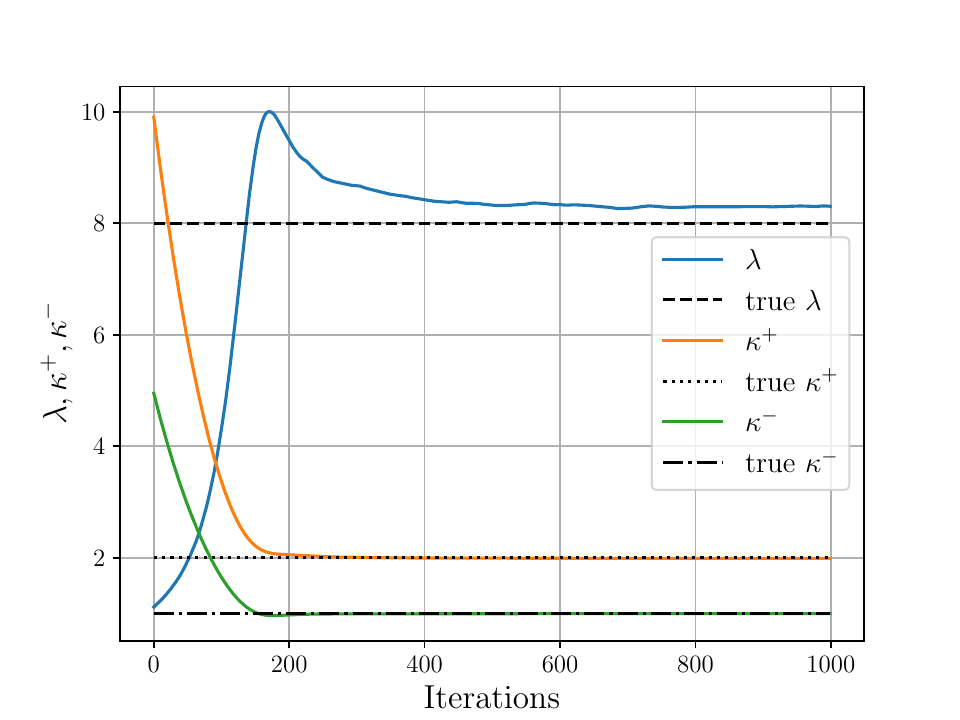}}
    \subfigure[Loss $\mu^\dagger$]{\includegraphics[width=0.48\linewidth]{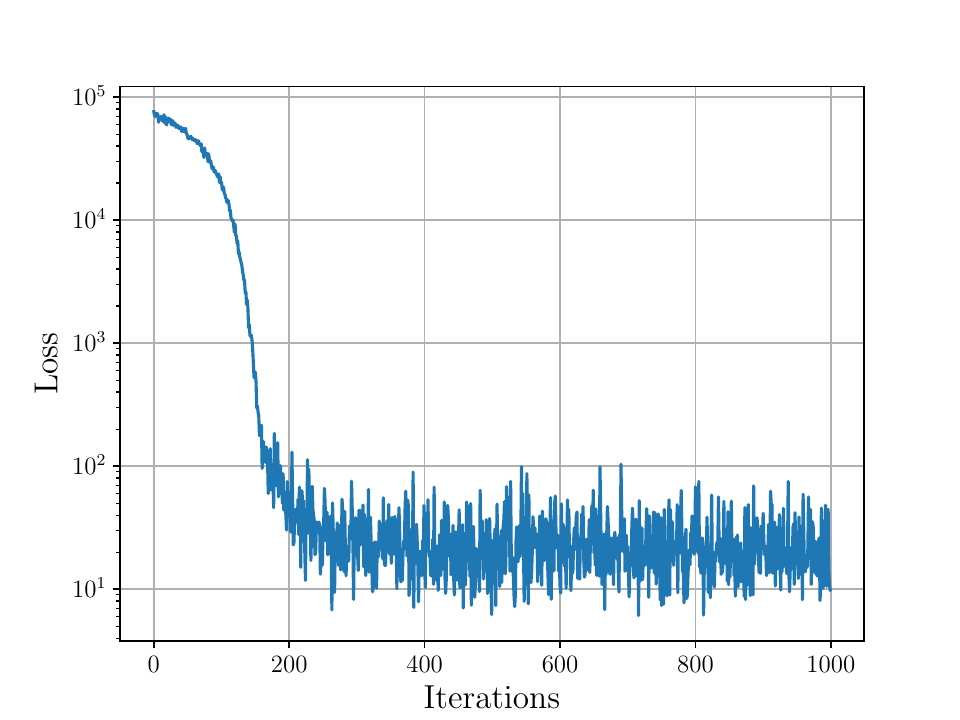}}\\
     \subfigure[Convergence $\tilde{\mu}^\dagger$]{\includegraphics[width=0.48\linewidth]{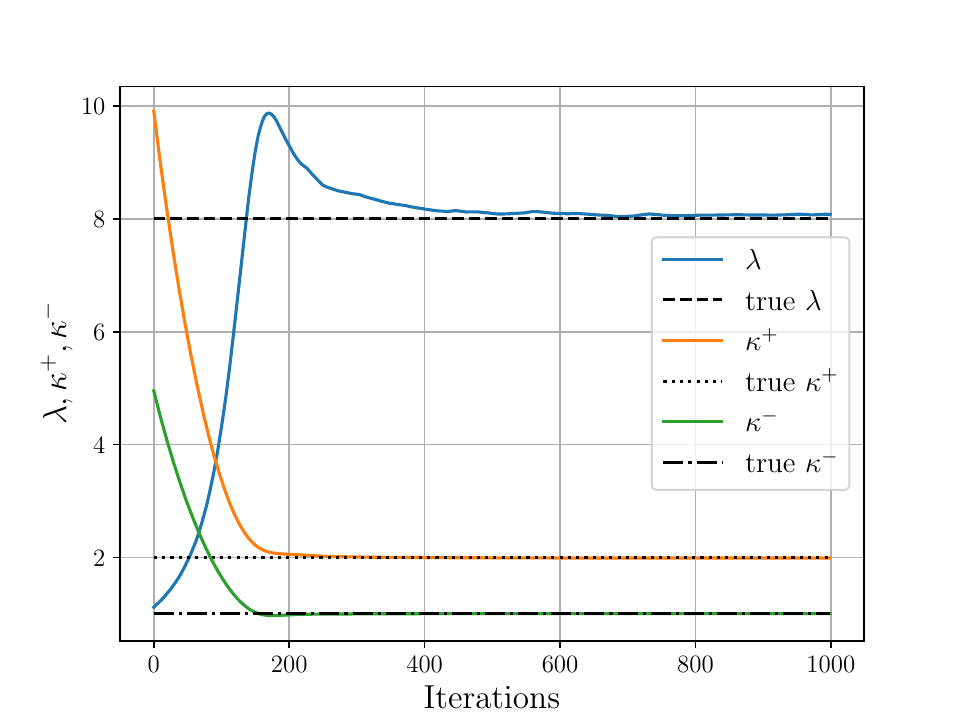}}
    \subfigure[Loss $\tilde{\mu}^\dagger$]{\includegraphics[width=0.48\linewidth]{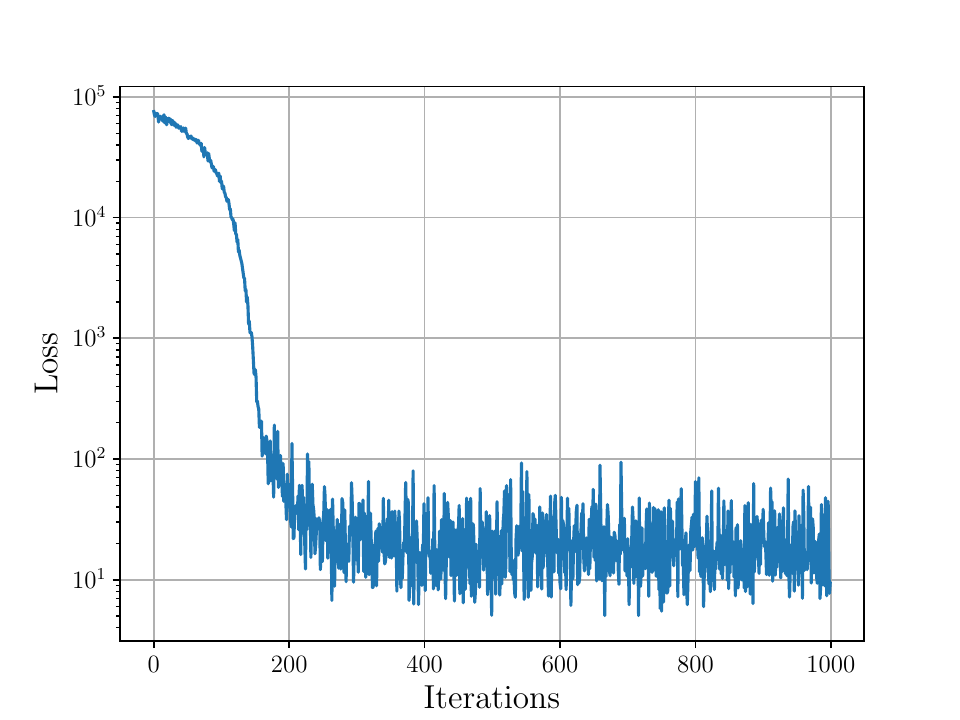}}
    \caption{Comparison of convergence of $\alpha=\{\kappa^\pm, \lambda\}$ for data generated from the physically realistic sharp $\mu^\dagger$ and the smoothed $\tilde{\mu}^\dagger$ for 1D Darcy with~\ref{eq:ideal_loss2}. We plot in (a) and (c) the convergence of the parameters themselves, and in (b) and (d) the loss function values.}
\end{figure}
We note that there is a symmetry in inference between $\kappa^+$ and $\kappa^-$, so after convergence we sort the $\kappa^\pm$ to associate $\kappa^+$ with the larger value and vice versa for $\kappa^-$. 
The number of physical systems from which we have data, $N$, is set to 1000 for the subsequent examples as this is shown to give very accurate results whilst highlighting the main benefit of the methodology -- to efficiently and accurately estimate underlying distributional parameters from large amounts of data.
        
With this experimental setup established, we can test the proposed methodology.
Subsections \ref{sssec:511} and \ref{sssec:512} concern the setting where we use a PDE
solve for the forward model, and hence minimize~\ref{eq:ideal_loss2}; we consider dimensions $d=1$
and $d=2$ respectively. In Subsections \ref{sssec:513} and \ref{sssec:514} we combine learning of $\alpha$
with operator learning to replace the PDE solve, and hence minimize~\ref{eq:ideal_loss3}, \ref{eq:ideal_loss3}; we again consider dimensions $d=1$ and $d=2$ respectively.

\subsubsection{Prior Calibration: 1D Darcy}
\label{sssec:511}
In this section we focus our attention on using~\ref{eq:ideal_loss2} to infer parameters of a prior $(T^\alpha)_{\#}\mu_0$ given an empirical measure $\nu^N$ and a forward operator $F^\dagger$.
\begin{figure}[t]\label{fig:convergence_1D_PriorCal_varBatch}
    \centering
    \subfigure[Convergence for $\lambda$]{\includegraphics[width=0.33\linewidth]{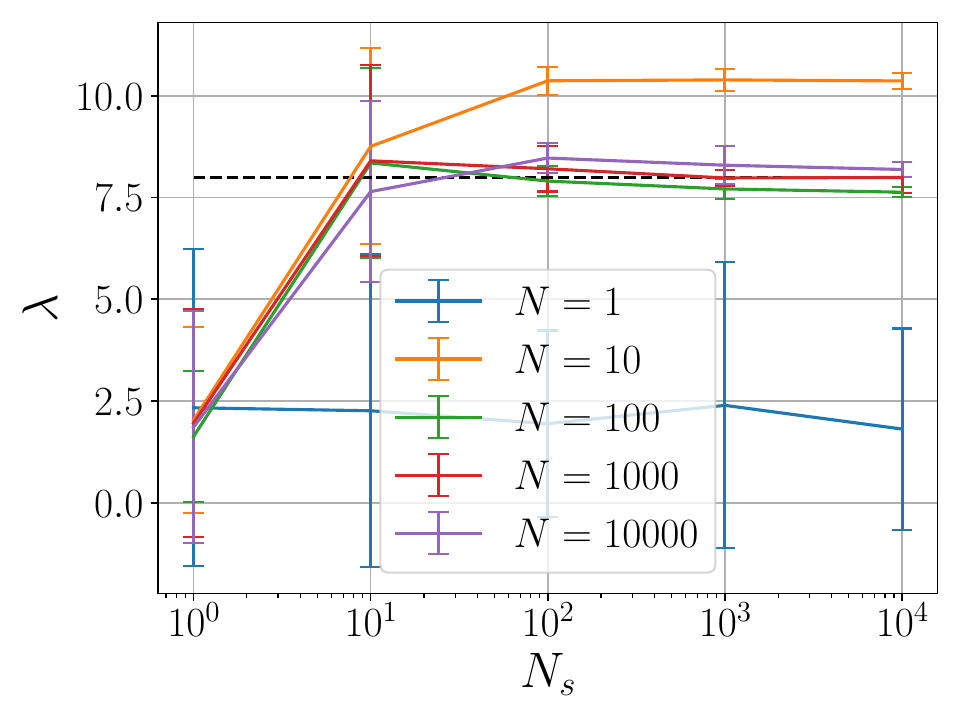}}
    \hspace{-0.5em}
     \subfigure[Convergence for $\kappa^+$]{\includegraphics[width=0.33\linewidth]{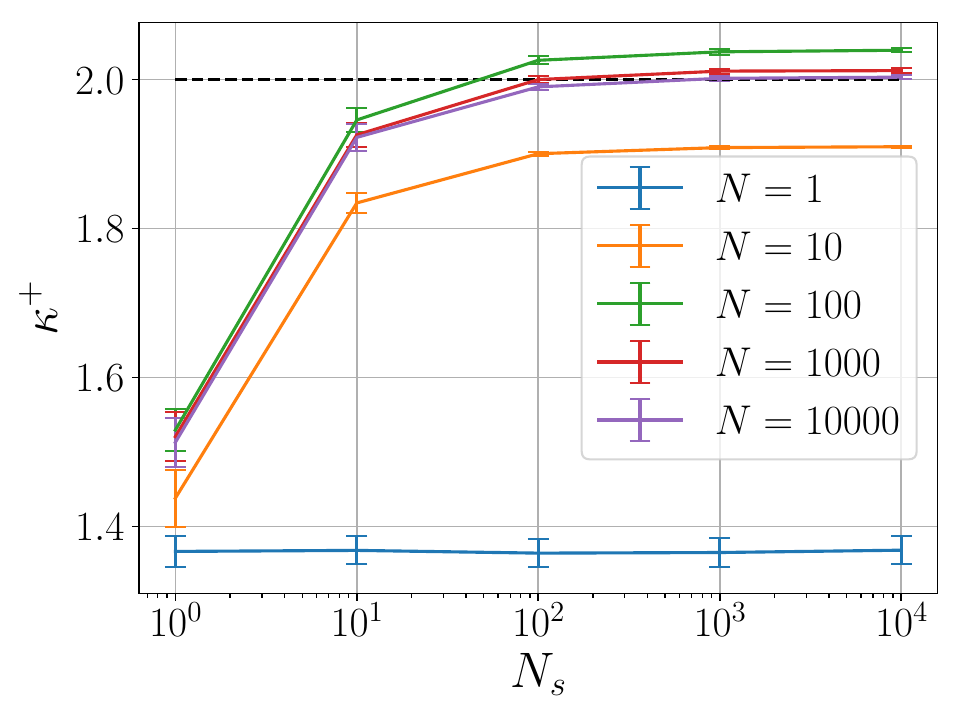}}
     \hspace{-0.5em}
     \subfigure[Convergence for $\kappa^-$]{\includegraphics[width=0.33\linewidth]{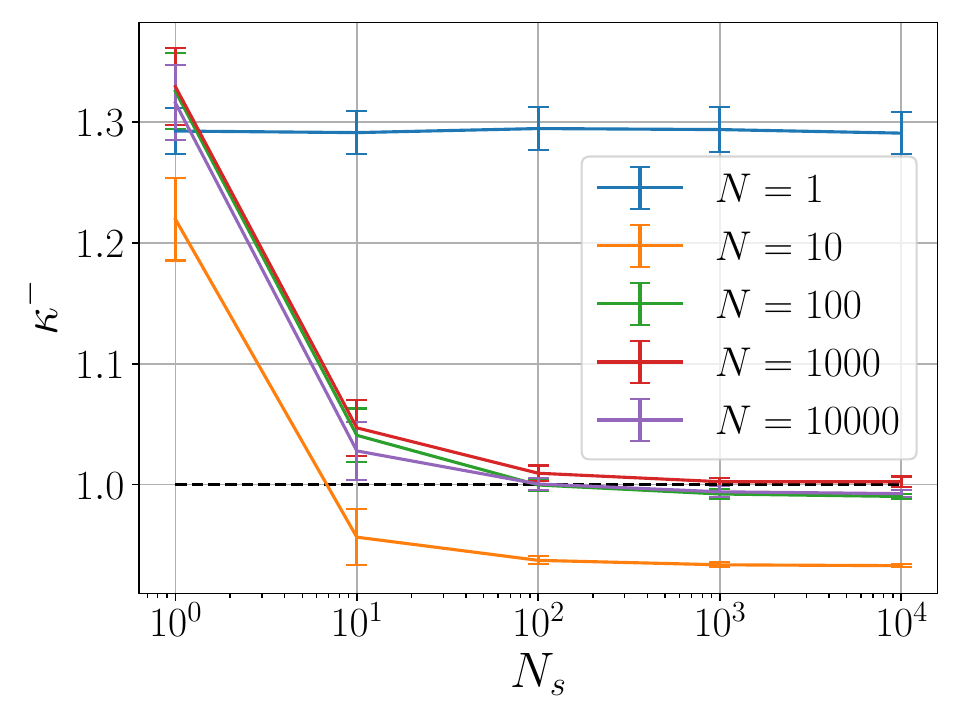}}
     \caption{Converged parameter estimation for $\lambda, \kappa^+, \kappa^-$ individually, for different dataset sizes $N$, and number of samples $N_s$ for~\ref{eq:ideal_loss2} on the 1D Darcy problem.}
\end{figure}
We use $N=1000$ data $y$ each of $d_y=50$ noisy pointwise observations of the solution with $\Gamma = 0.01^2\I$ observational noise covariance. The prior KL expansion is truncated at 20 terms. We use $N_s=1000$ samples to evaluate \ref{eq:ideal_loss2}. The true parameters for data generation are $\lambda=8, \kappa^+=2, \kappa^-=1$ and we set the level set smoothing parameter $\tau=10$.
We use a finite element mesh of 100 nodes with a weak form residual computed through array shifting and solved with conjugate gradients and the forcing function $f$ is set to a constant value of 10 for all examples. No regularizer $h(\alpha)$ is used for the 1D example.
We minimize using the smoothed model without sharp interfaces but we derive data both with the sharp and smoothed interfaces. In both cases we are able to accurately recover all three hyper
parameters in $\alpha$: see Fig.~\ref{fig:convergence_1D_PriorCal}. The runtime for these experiments is 41 seconds for the full 1k iterations. The relative error of the predicted parameters for the smoothed level set prior are $0.56\%$, $0.28\%$, and $0.96\%$ for $\kappa^+$, $\kappa^-$, and $\lambda$ respectively. We observe that the mis-specification between $\mu^\dagger$ (draws are discontinuous) and $(T^\alpha)_{\#}\mu_0$
(draws are smooth) is reflected in a small inference bias, as would be expected.
In Fig.~\ref{fig:convergence_1D_PriorCal_varBatch} we plot the prior parameter estimation  mean and standard deviations for 100 random initializations, varying the number of samples $N_s$ used for estimating the loss functions and the size of the dataset ($N$). The $\alpha$ 
parameters are initialized from $\log \lambda \sim \mathrm{Unif}(\log 0.5, \log 4)$, $\log \kappa^- \sim \mathrm{Unif}(\log 0.5, \log 4)$, and $\log \kappa^- \sim \mathrm{Unif}(\log 6, \log 10)$. We observe that the mean of the recovered parameters roughly converges for this problem setup after a dataset size $N=100$ and number of samples $N_s=100$. 

\subsubsection{Prior Calibration: 2D Darcy}
\label{sssec:512}
In this section we generalize the setting of the previous section to the two-dimensional domain $D = [0, 1]\times[0,1]$. As in the 1D example, we use $N=1000$ data $y$ each of $d_y=50$ noisy pointwise observations of the solution with $\Gamma = 0.01^2\I$ observational noise covariance. The prior KL expansion is truncated at 20 terms in each dimension, for a total of 400 terms. We use $N_s=100$ samples to evaluate \ref{eq:ideal_loss2}. The true parameters for data generation are $\lambda=5, \kappa^+=2, \kappa^-=1$ and we set the level set smoothing parameter $\tau=5$.
We use a finite element mesh of $100\times100$ nodes. The regularizer $h(\alpha)$ in \ref{eq:ideal_loss2} has means $m_{h, \lambda}=\log(10)$ and $m_{h, \kappa^\pm}=\log(3)$, with standard-deviations $\sigma_h=2$.
In Fig.~\ref{fig:2D_diffusion_fields_and_slns}, two permeability fields associated with their sets of observations from the dataset are shown. 
\begin{figure}[]\label{fig:2D_diffusion_fields_and_slns}
    \centering
    \subfigure[$z$ field]{\includegraphics[width=0.32\linewidth]{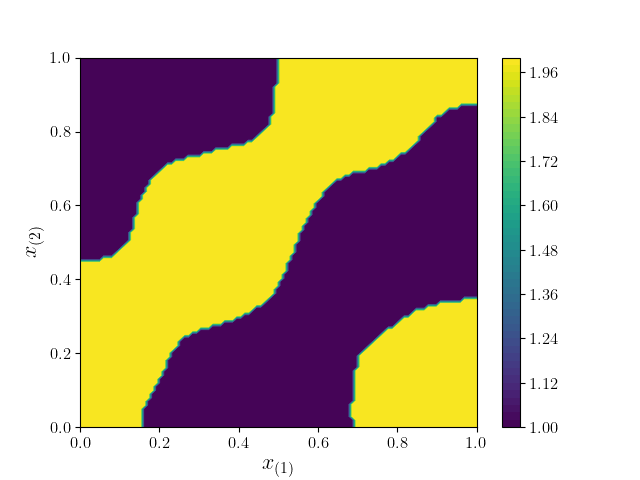}}
    \hspace{-1.5em}
    \subfigure[Smoothed $z$ field]{\includegraphics[width=0.32\linewidth]{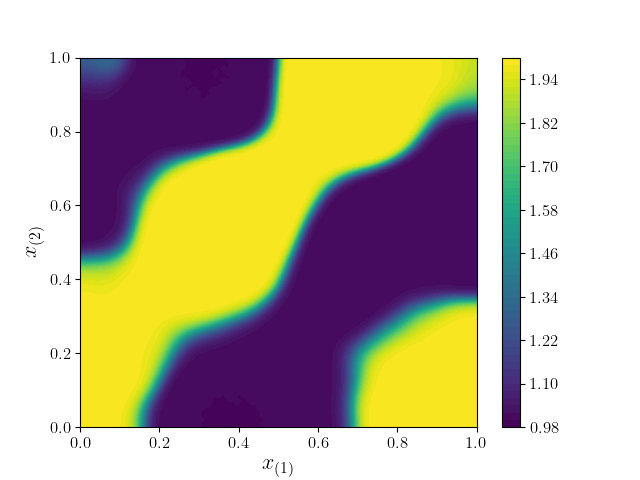}}
    \hspace{-1.5em}
    \subfigure[$u$ and 50 observations]{\includegraphics[width=0.32\linewidth]{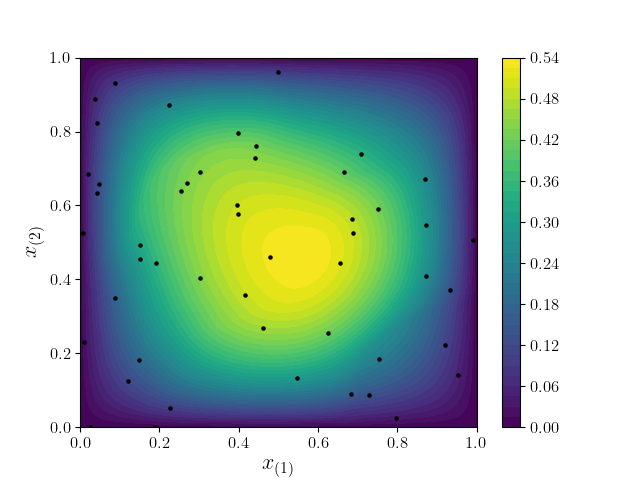}}
    \caption{Plots of a sharp permeability fields and the smoothed field and PDE solutions with observation locations in the dataset for the 2D Darcy problems 
    We note that the solutions are computed from the sharp level sets for the dataset, whereas $(T^\alpha)_{\#}\mu_0$ in the loss function uses the smoothed level sets to make use of gradient-based optimization.}
\end{figure} 
\begin{figure}[]\label{fig:convergence_2D_PriorCal}
    \centering
     \subfigure[$\alpha$ convergence]{\includegraphics[width=0.48\linewidth]{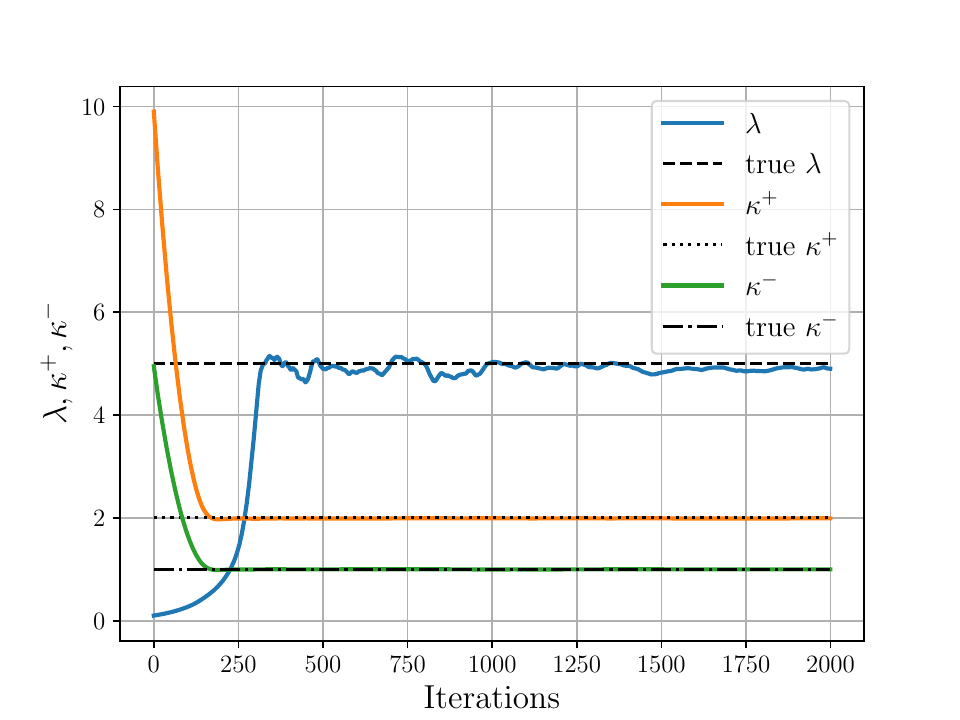}}
     \subfigure[Losses]{\includegraphics[width=0.48\linewidth]{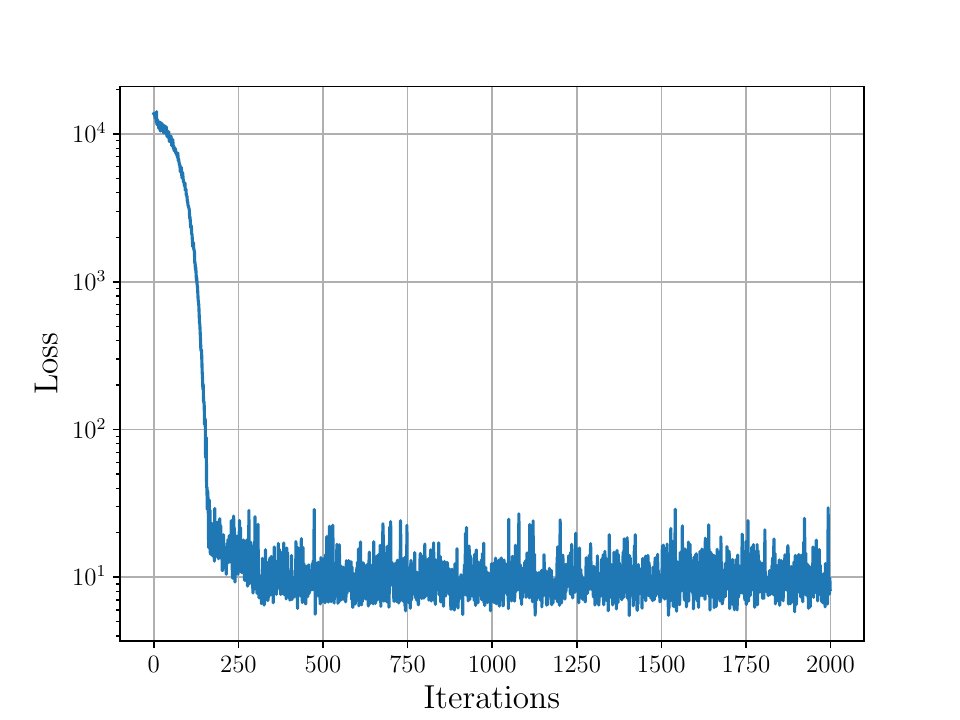}}
    \caption{Convergence of prior parameters $\alpha$ for 2D Darcy trained with~\ref{eq:ideal_loss2}.}
\end{figure}
Fig.~\ref{fig:convergence_2D_PriorCal} shows the convergence plots for prior parameter calibration. Inference time was 1.52 minutes for the 2k iterations. The relative error of the predicted parameters are $0.39\%$, $0.03\%$, and $1.96\%$ for $\kappa^+$, $\kappa^-$, and $\lambda$ respectively.

\subsubsection{Prior Calibration and Operator Learning: 1D Darcy}
\label{sssec:513}
\begin{figure}[t]\label{fig:convergence_1D_Prior_Calib_Op}
    \centering
    \subfigure[$\alpha$ convergence]{
    \includegraphics[width=0.48\linewidth]{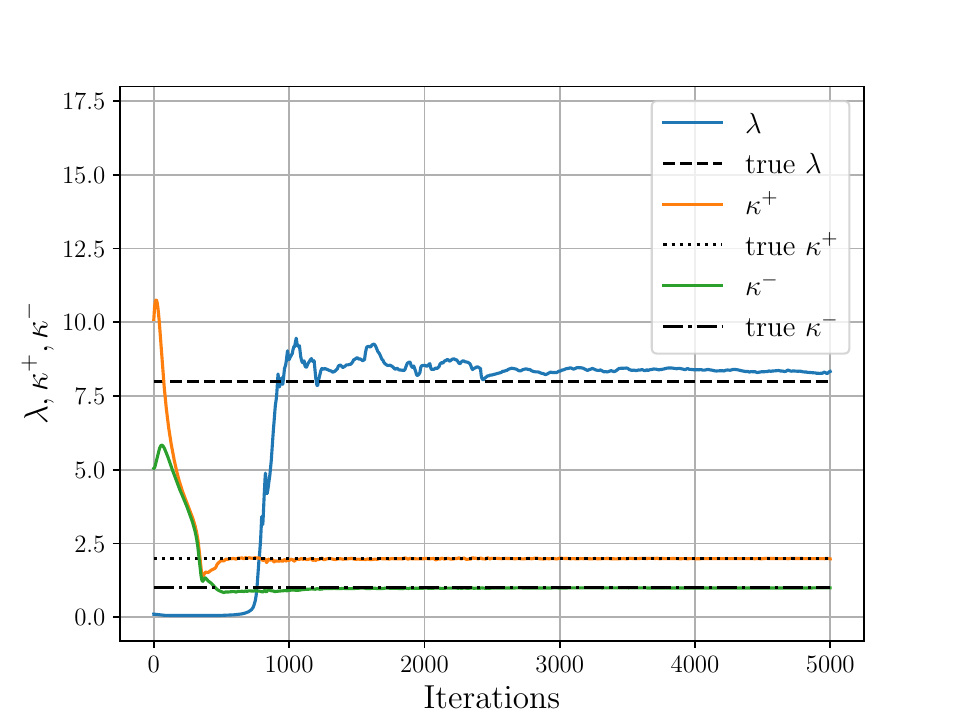}}
    \subfigure[Losses]{\includegraphics[width=0.48\linewidth]{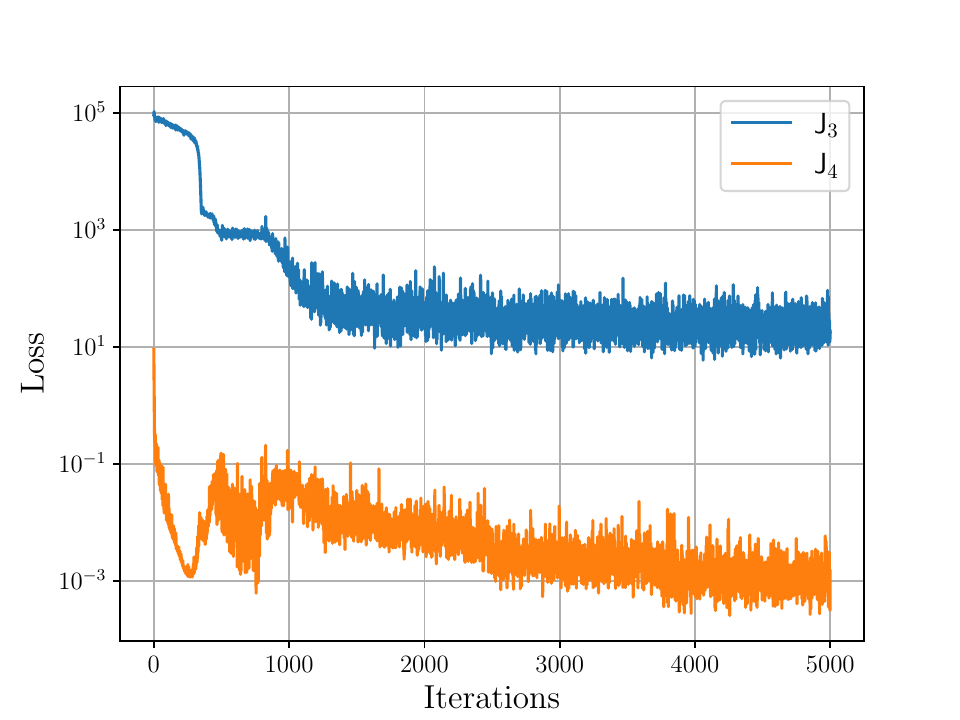}}
    \caption{Convergence of prior parameters $\alpha$, with learned operator $F^\phi$ on 1D Darcy with~\ref{eq:ideal_loss3} and \ref{eq:ideal_loss4}. We take 10 parameter updates steps on $\phi$ in the lower objective of the bilevel optimization for every step on $\alpha$. }
\end{figure}
We now turn our attention to the task of using~\ref{eq:ideal_loss3} and \ref{eq:ideal_loss4} to jointly estimate $(T^\alpha)_{\#}\mu_0$ and learn a parametrized operator $F^\phi$. We choose this operator to be a 4-layer Fourier Neural Operator~\cite{li2020fourier} with 64 dimensional channel space and 8 Fourier modes and Silu activation functions~\cite{hendrycks2016gaussian} ($\mathbf{\sigma}(x)=x\,\mathrm{sigmoid}(x)$). The last layer performs a point-wise multiplication with a function that is zero on the boundary ($\sin( \pi x)$). \ref{eq:ideal_loss4} is evaluated with $N_r=20$ samples and~\ref{eq:ideal_loss3} is evaluated with $N_s=1000$ samples. We perform 10 update steps on $\phi$ per $\alpha$ update in the lower-level optimization routine. We use the same data setup, prior, and regularizer as in Subsection~\ref{sssec:511}.
In Fig.~\ref{fig:convergence_1D_Prior_Calib_Op} we show that we are able to jointly learn the neural operator
approximation of the forward model and the unknown parameters. The runtime is 2.44 minutes. The relative error for the FNO predictions against the solver for the data samples $z\sim\mu^\dagger$ from the true sharp level set prior is 0.40\%. The relative error of the predicted parameters are $1.13\%$, $0.71\%$, and $4.18\%$ for $\kappa^+$, $\kappa^-$, and $\lambda$ respectively.
\subsubsection{Prior Calibration and Operator Learning: 2D Darcy}
\label{sssec:514}
\begin{figure}
    \centering
    \subfigure[$\alpha$ convergence with $\phi^\star(\alpha)$]{\includegraphics[width=0.48\linewidth]{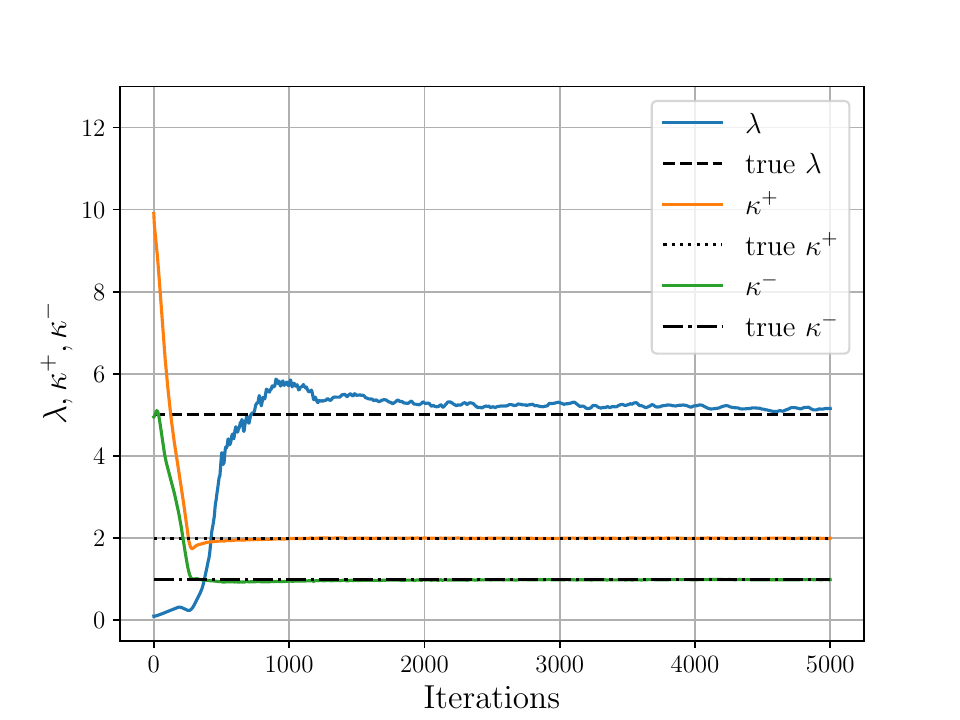}}
    \subfigure[Losses]{\includegraphics[width=0.48\linewidth]{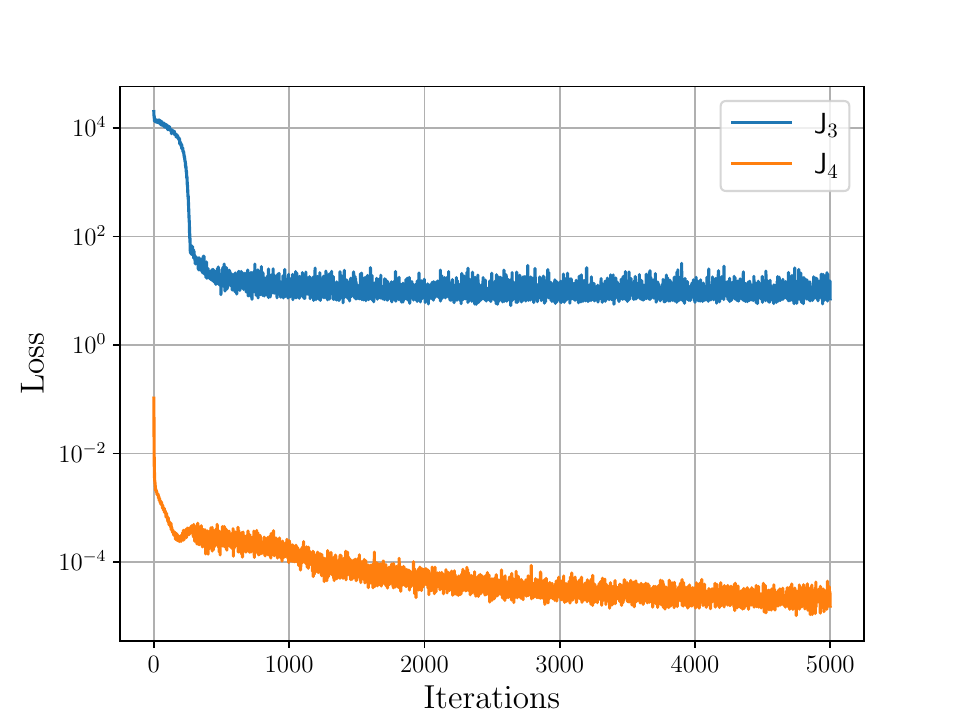}}
    \caption{Convergence of prior parameters $\alpha$ for the level set prior on the 2D Darcy problem with jointly learned operator $F^\phi$. For every step of prior update $\alpha$, 10 parameter updates are performed on $\phi$.}
    \label{fi_convergence_prior_op}
\end{figure}
We now apply the joint prior-operator learning methodology to the 2D setting, keeping the same FNO parameters as in 1D case (Subsection ~\ref{sssec:513}) but for 2 dimensional input fields. The last layer of the FNO is now $\sin(\pi x_{(1)})\sin(\pi x_{(2)})$. We use the same data setup, prior, and regularizer as in Subsection~\ref{sssec:512} and $N_r=20$, $N_s=100$.
Figure~\ref{fi_convergence_prior_op} shows the convergence for the learning of both $(T^\alpha)_{\#}\mu_0$ and $F^\phi$. The model converged accurately after a runtime of 56.46 minutes. The relative error for the FNO predictions against the solver for the data samples $z\sim\mu^\dagger$ from the true sharp level set prior is of 0.73\%. The relative error of the predicted parameters are $0.10\%$, $0.98\%$, and $3.19\%$ for $\kappa^+$, $\kappa^-$, and $\lambda$ respectively.

\subsection{Lognormal Priors}
\label{ssec:logn}
In this subsection we focus on learning parameters of a lognormal prior. We introduce a 
Mat\'ern-like field $a$ with regularity $\nu$, amplitude $\sigma$, and length-scale $\ell$; to be specific
we generate a centred Gaussian random field (GRF) with covariance operator
\begin{align}
    C_{\sigma,\ell, \nu} = \gamma\ell^d(-\ell^2\Delta + \I )^{-\nu-d/2},\nonumber
\end{align}
setting
\begin{align}
    \gamma = \sigma^2\frac{2^d\pi^{d/2}\Gamma(\nu+d/2)}{\Gamma(\nu)},\nonumber
\end{align}
where $\Gamma(\cdot)$ is the Gamma function (not to be confused with our use of $\Gamma$ as the observational noise covariance).
In practice we truncate a Karhunen-Loeve expansion in the form
\begin{align}
    a(x;\sigma, \ell, \nu) = \sum_{j=1, k=1}^{J, K} \sqrt{\gamma\ell^d\left(\ell^2(j^2+k^2)\pi^2+1\right)^{-\nu-d/2}}\varepsilon_{j,k}\varphi_{j,k}(x), \;\; \varepsilon_{j,k}\sim \NPDF(0, 1)\mathrm{i.i.d.}\nonumber
\end{align}
We then set $z=\exp(a).$
Again, because of the countable nature of the construction of the lognormal prior, it lies in
a separable subspace $Z$ of $L^\infty(D).$
We focus on estimating $\alpha=\{\nu, \ell\}$ which represent the regularity and length-scale of the lognormal permeability field $z$. We do not attempt to jointly infer the amplitude $\sigma$ as this induces a lack of identifiability of parameters at the level of the Karhunen-Lo\`eve expansion spectrum. We note that, as the amplitude parameter does not contain much spatially dependent information,  it can be accurately estimated by other means, hence here it set to $\sigma=1$. We attempt to jointly learn $\alpha=\{\nu, \ell\}$
and $\phi$ the parameters of a residual-based neural operator approximation. In Subsection \ref{sssec:521}
we consider a setting where the regularity and lengthscale parameters are identifiable and show successful
joint learning of $(\alpha,\phi).$ However it is intuitive that the regularity and lengthscale may not
be separately identifiable; we show in Subsection \ref{sssec:522} that in such situations the entanglement
of jointly learning $(\alpha,\phi)$ can cause convergence problems with the proposed methodology.

\begin{figure}
    \centering
    \subfigure[$a$]{\includegraphics[width=0.32\linewidth]{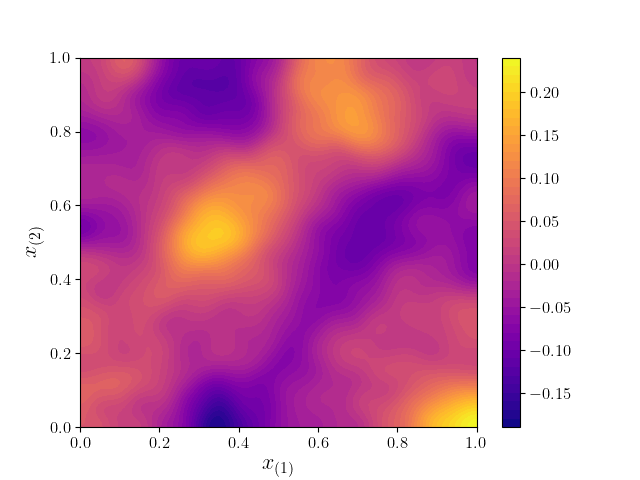}}
    \subfigure[$z$]{\includegraphics[width=0.32\linewidth]{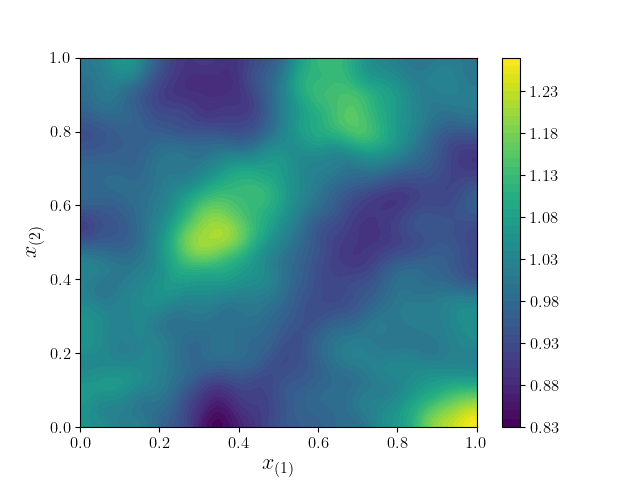}}
    \subfigure[$u$ and 50 observations]{\includegraphics[width=0.32\linewidth]{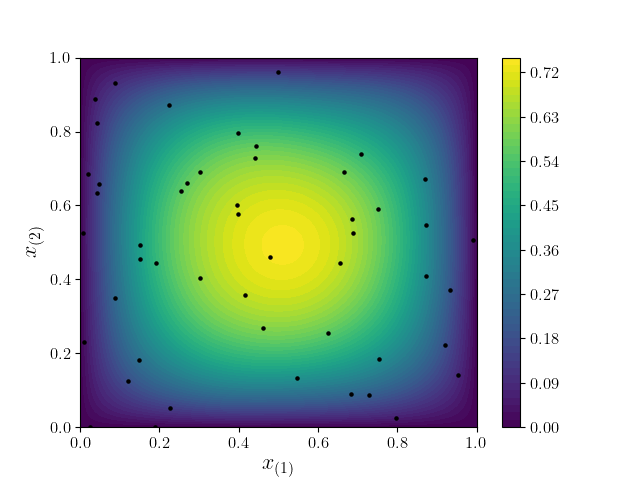}}
    \caption{(a) A sampled GRF $a$, (b) the exponentiated random field $z$, (c) the solution field associated to $z$ with the 50 observation locations.}
    \label{fig:WM_samples}
\end{figure}
\begin{figure}
    \centering
    \subfigure[$\alpha$ convergence]{\includegraphics[width=0.48\linewidth]{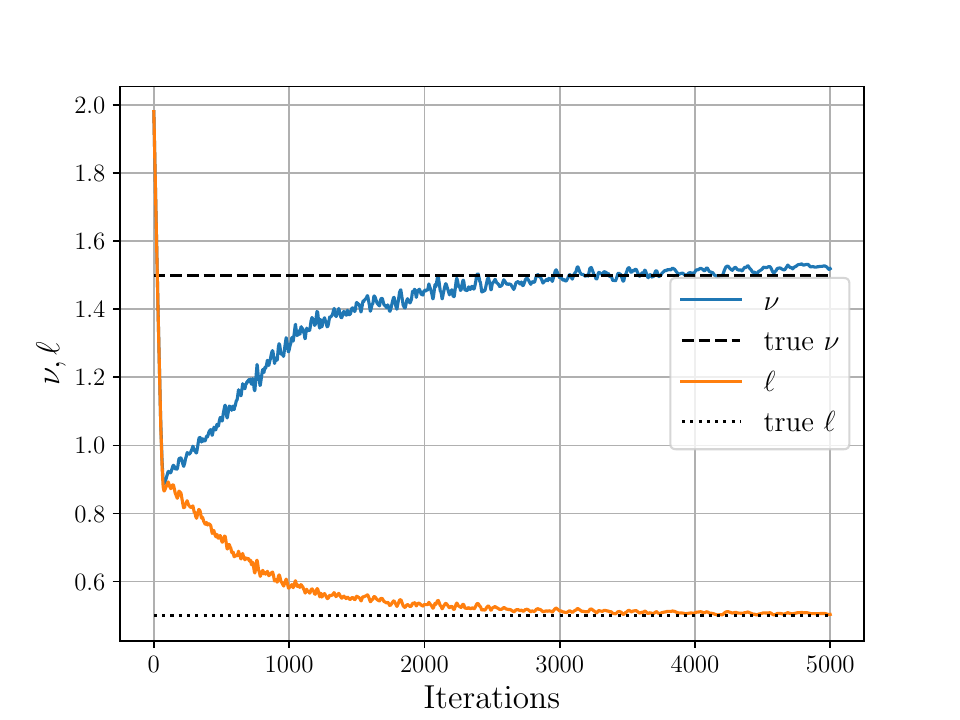}}
    \subfigure[Loss]{\includegraphics[width=0.48\linewidth]{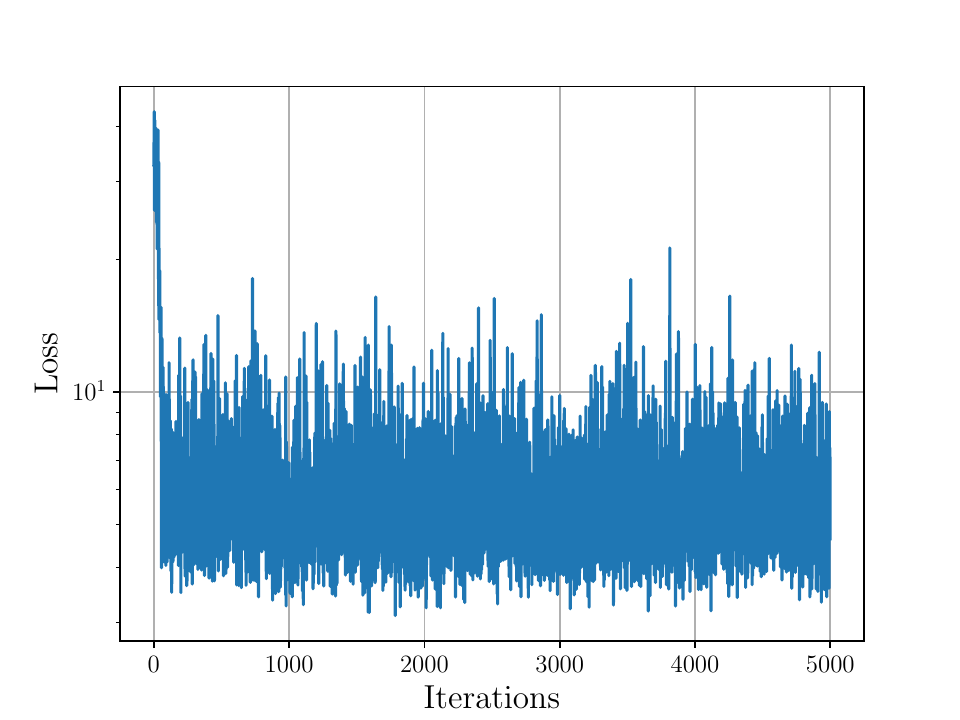}}
    \subfigure[$\alpha$ convergence with $\phi^\star(\alpha)$]{\includegraphics[width=0.48\linewidth]{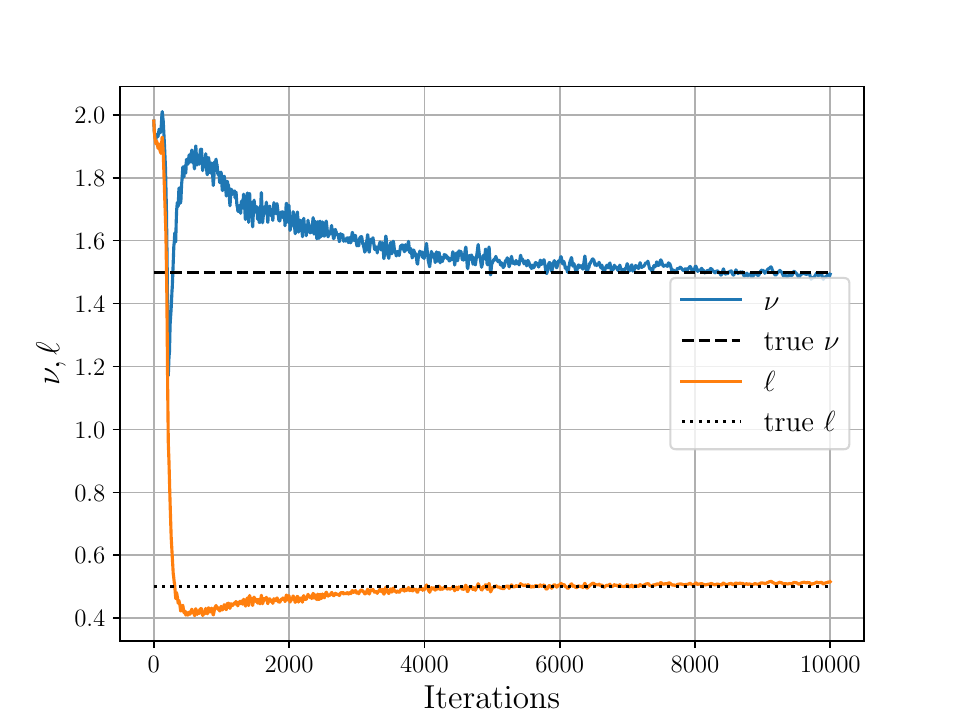}}
    \subfigure[Losses]{\includegraphics[width=0.48\linewidth]{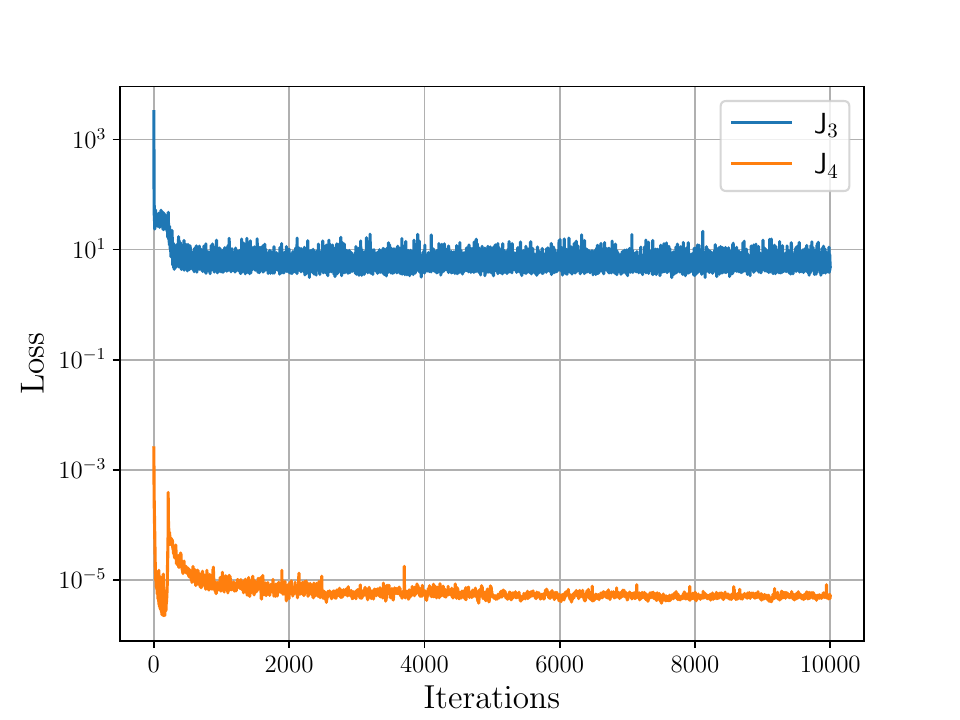}}
    \caption{(a)-(b) Convergence and loss for $\alpha=\{\nu, \ell\}$ on the 2D Darcy problem for the lognormal prior. (b)-(c) Convergence of combinations of prior parameters $\alpha=\{\nu, \ell\}$ estimated jointly with $\phi$ for 2D Darcy on the lognormal prior with 10 $\phi$ updates in the lower-level optimization for every $\alpha$ step.}
    \label{fig:2DDarcy_convergence_prior_WN}
\end{figure}
\subsubsection{Identifiable Setting}\label{sssec:521}

We use $N=1000$ data $y$ each of $d_y=50$ noisy pointwise observations of the solution with $\Gamma = 0.01^2\I$ observational noise covariance. The prior KL expansion is truncated at 20 terms in each dimension for a total of 400 expansion terms. We set $N_s=100$. The true parameters for data generation are $\nu=1.5, \ell=0.5$.
We use a finite element mesh of $100\times100$ nodes. The regularizer $h(\alpha)$ in \ref{eq:ideal_loss3} has means $m_{h, \nu}=\log(3.5)$ and $m_{h, \ell}=\log(1)$, with standard-deviations $\sigma_h=2$.
In Fig.~\ref{fig:2DDarcy_convergence_prior_WN}(a, b) we show the convergence plots for the prior only learning. 
The relative error on parameter estimation is $1.26\%$ and $0.66\%$ on $\nu$ and $\ell$ respectively. The  runtime is of 4.26 mins.

We then test the joint estimation of $\alpha$ and $\phi$ with \ref{eq:ideal_loss3}, \ref{eq:ideal_loss4}.
 The data setup, prior, and regularizers are the same as in the $\alpha$-only case described
 in Subsection \ref{sssec:512} with $N_r=20$, $N_s=100$. The operator setup is the same as in Subsection \ref{sssec:514}.
In Fig.~\ref{fig:2DDarcy_convergence_prior_WN}(c, d) we show the joint prior learning with operator learning.  The achieved relative error for the learned $F^\phi$ on $z\sim\mu^\dagger$ is $0.128\%$.
The relative error on parameter estimation is $0.44\%$ and $3.13\%$ on $\nu$ and $\ell$ respectively. The  runtime is of $112$ mins for the 10k iterations.

\subsubsection{Unidentifiable Setting}
\label{sssec:522}
\begin{figure}
    \centering
    \subfigure[$\alpha$ convergence]{\includegraphics[width=0.48\linewidth]{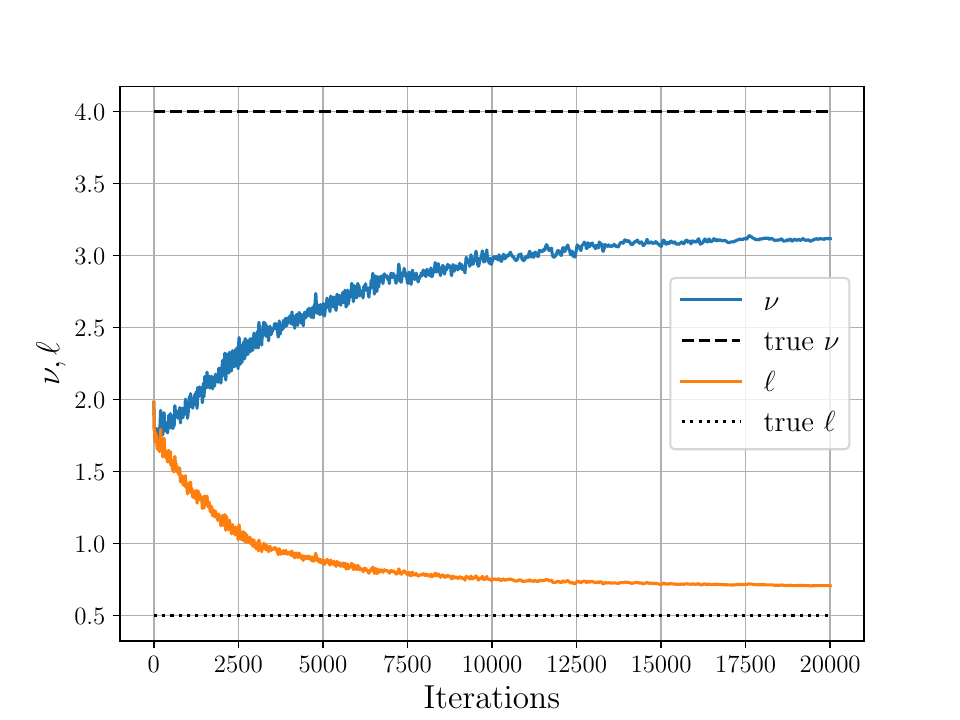}}
    \subfigure[Loss]{\includegraphics[width=0.48\linewidth]{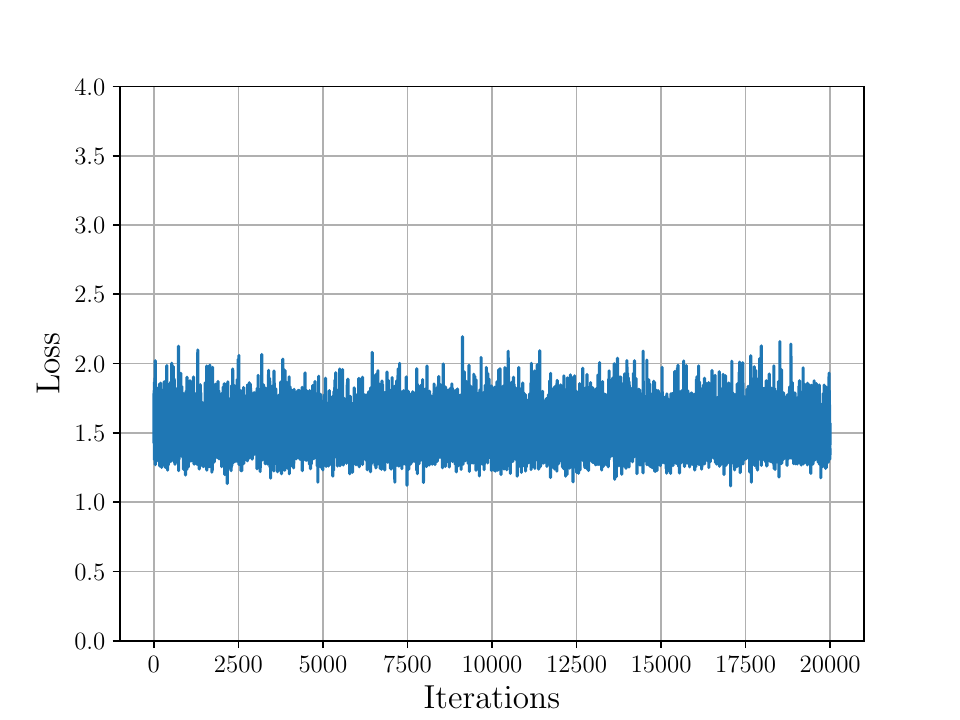}}
    \\
    \subfigure[$\alpha$ convergence with $\phi^\star(\alpha)$]{\includegraphics[width=0.48\linewidth]{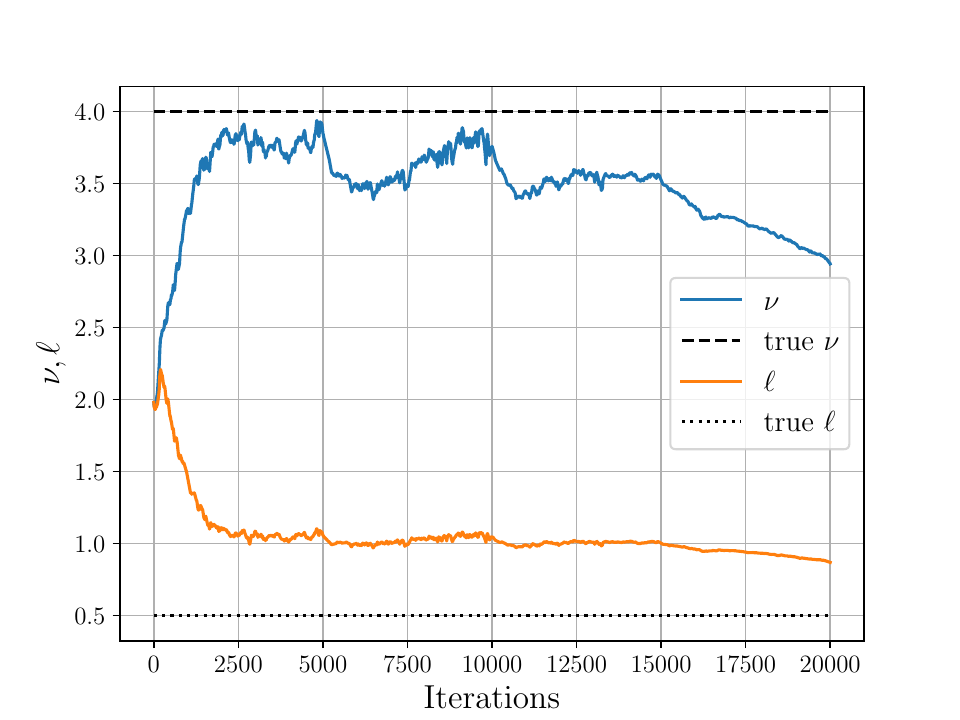}}
    \subfigure[Losses]{\includegraphics[width=0.48\linewidth]{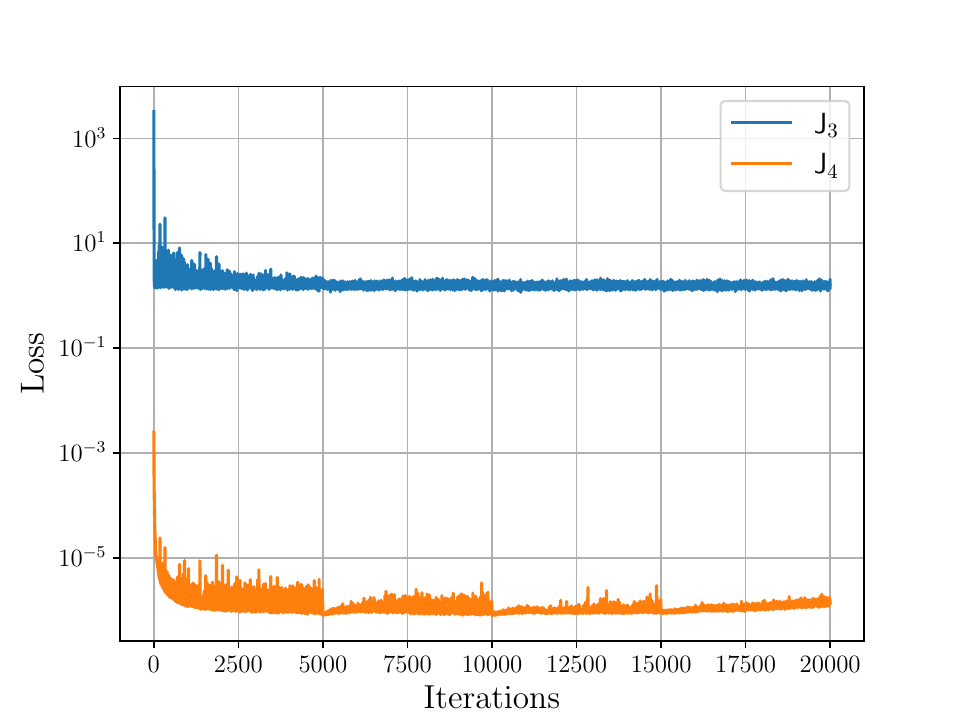}}
    \caption{(a)-(b) Convergence  and loss for $\alpha$ learning. (c)-(d) Convergence  and losses for $\alpha$ and $\phi$ learning. Both sets of figures are for the 2D Darcy on log normal prior for unidentifiable $\alpha$ parameter values.}
    \label{fig:2DDarcy_convergence_prior_op_WN_unident}
\end{figure}
For certain parameter values of $\alpha$, from which the true data is generated, there can be a lack of identifiability. In the previous section, the data was generated from a relatively rough random field with a relatively long length-scale. Hence the smoothness and length-scale parameters each have a distinct impact on the prior's spectrum. We will now look at a smooth random field setting $\nu=4$, keeping $\ell=0.5$. Now, both parameters will have a similar influence on the spectrum of the prior. Other parameters for the data, solver/operator, and regularizer are kept the same as in Subsection \ref{sssec:521}. In Fig.~\ref{fig:2DDarcy_convergence_prior_op_WN_unident} we show the convergence of the $\alpha$ parameters for the unidentifiable parameter regime for the $\alpha$ only learning and the joint $\alpha, \phi$ prior and operator learning tasks. We can see the lack of identifiability causes more issues in the operator learning as it struggles to gain any inferential traction. For these examples, it was necessary to run the learning models for more iterations (20k) and use a 6-time decaying learning rate on $\alpha$ (as opposed to the usual 4 in previous examples). Furthermore, due to training instability, the $\phi$ parameters had to be trained with AMSGrad~\cite{reddi2019convergence}(with the same settings) instead of Adam. As we can see from the Fig.~\ref{fig:2DDarcy_convergence_prior_op_WN_unident}, the proposed methodology applied to this particular unidentifiable setup does not fail silently, the $\alpha$ only learning has a very long convergence time, and the $\alpha$ and $\phi$ learning does not reach equilibrium. Practitioners may also find the method is more sensitive to initializations in such unidentifiable settings. Other possible avenues to identifying non-uniqueness in prior parameter estimation include assessing the sensitivity of the estimated parameters on the choice of regularizer mean and standard-deviation, or the use of hierarchical parametric priors where a distribution on $\alpha$ is learned. There, a wide (or multimodal) inferred distribution on $\alpha$ could indicate non-identifiability of the prior parameters as well as the covariance between entangled prior parameters.

\section{Conclusions}
\label{sec:C}
In this work we propose a novel methodology for learning a generative model for a prior on function
space, based on indirect and noisy observations defined through solution of a PDE. The learning scheme is based on the minimization of a loss functional computing divergences at a distributional level. We prove our methodology recovers the Bayesian posterior when observations originate from a single physical system ($N=1$). We demonstrate the accuracy of our methodology on a series of example pertinent for practitioners. These are 1D and 2D steady-state Darcy flown problems for two different parametric priors. Namely, a level set prior and a lognormal prior. Furthermore, we show how the proposed framework can be augmented to jointly learn an operator, mapping samples of PDE parameter fields drawn from the estimated prior, and the solution space of the PDE, through a bilevel optimization strategy. The operator learning takes place on-the-fly during the
optimization of the prior parameters, and is residual-based. Finally, we indicate the possible pitfalls of unidentifiable priors and show the emergent behavior of the methodology under this setting. A selection of avenues for future work include, the use of different metrics on the space of measures, testing on different PDE systems (nonlinear, coupled, time-evolving etc.), other choices of parametrized prior measures, studying identifiability and non-uniqueness of inferred prior parameters, and testing the downstream use of such learned priors and operators as in Bayesian inversion and generative modeling of physical systems. 

\appendix
\section{Proofs}
\subsection{Proof of Lemma~\ref{lemma:pushW}}\label{app:sec:pushW}
\begin{proof}
Let $\nu,\mu$ denote two probability measures defined on a common measure space 
    $(X, \Sigma).$ Let $\Pi(\nu, \mu)$ denote the set of all coupling probablity measures $\gamma$, on the product space $X\times X$, such that $\gamma(A \times X)= \nu(A)$, $\gamma(X \times A)= \mu(A)$ for all $A\in \Sigma$. Remembering $P_\B(\cdot)=\B^{-1/2}\,\cdot\,$, now recall the definition
    \begin{align}
        \mathsf{W}^2_{2, \B}(\nu, \mu) = \inf_{\gamma \in \Pi(\nu, \mu)} \int_{\Rb^d \times \Rb^d} \|x - y\|_{\B}^2 \md \gamma(x, y).
    \end{align}
The set of all $\gamma' \in \Pi(P_{\B\#}\nu, P_{\B\#}\mu)$ is equivalent to the set
    $(P_{\B}\otimes P_\B)_\#\gamma$ defined over all $\gamma\in\Pi(\nu, \mu)$.
Thus we have
    \begin{align}
    \mathsf{W}^2_{2}(P_{\B\#}\nu, P_{\B\#}\mu) &= \inf_{\gamma' \in \Pi(P_{\B\#}\nu, P_{\B\#}\mu)} \int_{\Rb^d \times \Rb^d} \|x' - y'\|_{2}^2\, \md \gamma'(x', y')\\
        &= \inf_{\gamma \in \Pi(\nu, \mu)} \int_{\Rb^d \times \Rb^d} \|x - y\|_{2}^2 \,\md((P_\B\otimes P_\B)_\#\gamma(x, y))\nonumber\\
        &= \inf_{\gamma \in \Pi(\nu, \mu)} \int_{\Rb^d \times \Rb^d} \|\B^{-\frac{1}{2}}x - \B^{-\frac{1}{2}}y\|_{2}^2 \,\md \gamma(x, y)\nonumber\\
        &= \inf_{\gamma \in \Pi(\nu, \mu)} \int_{\Rb^d \times \Rb^d} \|x - y\|_{\B}^2 \,\md \gamma(x, y)\nonumber\\
        &=\mathsf{W}^2_{2, \B}(\nu, \mu).
    \end{align}
    
\end{proof}
\clearpage
\subsection{Proof of Lemma~\ref{lemma:weightedW}}\label{app:sec:weightedW}
\begin{proof}
To show the desired result we determine a point at which the infimum over couplings $\Pi(\delta_y, \mu)$
is achieved. It is known from~\cite{santambrogio2015optimal}[Section \S 1.4] that when one of the measures in the argument of a Kantorovich problem (KP) (of which Wasserstein metrics are a special case) is a Dirac, i.e. $\delta_y$ for any $y\in\bR^d$, then the set of couplings $\Pi(\delta_y, \mu)$ with marginals $\delta_y$, $\mu$, contains a single element, namely $\delta_y \otimes \mu$. It follows that
\begin{align}
\mathsf{W}^2_{2, \Gamma}(\delta_y, \mu) &= \inf_{\gamma \in \Pi(\delta_y, \mu)} \int_{\Rb^d \times \Rb^d} \|y' - x\|_\Gamma^2 \mathrm{d}\gamma(y', x) \nonumber \\
&= \int_{\Rb^d \times \Rb^d} \|  y' -  x\|_{\Gamma}^2 (\md \delta_y(y') \otimes \md\mu(x)) \nonumber \\
&= \Eb_{x \sim \mu} \|y-x\|_{\Gamma}^2.
\end{align}
\end{proof}

\subsection{Proof of Lemma \ref{lemma:w_sw_equiv}}\label{app:w_sw_equiv}
\begin{proof}
    We want to show, for $y\in\Rb^d$ and $\mu \in \mathcal{P}(\Rb^d)$,
\begin{align*}
    \mathsf{SW}_{2,\Gamma}^2(\delta_{y}, \mu)=\frac{1}{d}  \mathsf{W}_{2,\Gamma}^2(\delta_{y}, \mu).
\end{align*}
Thus, we have
\begin{subequations}
\begin{align}
    \mathsf{SW}^2_{2, \Gamma}(\delta_y, \mu) &= \int_{\Sb^{d-1}}\mathsf{W}_2^2(P^\theta_{\Gamma\#}\delta_y, P^\theta_{\Gamma\#}\mu)\md \theta \nonumber\\
    &=\int_{\Sb^{d-1}}\mathsf{W}_2^2(\delta_{\langle\Gamma^{-1/2}y, \theta\rangle}, P^\theta_{\Gamma\#}\mu)\md \theta \nonumber\\
    &=\int_{\Sb^{d-1}} \Eb_{z\sim\mu}( \langle\Gamma^{-1/2}y, \theta\rangle -  \langle\Gamma^{-1/2} z, \theta\rangle )^2\md \theta \label{eq:pre_swap} \\
    &= \Eb_{z\sim\mu} \int_{\Sb^{d-1}}( \langle\Gamma^{-1/2}y, \theta\rangle -  \langle\Gamma^{-1/2} z, \theta\rangle )^2\md \theta \label{eq:swapped_integral} \\
    &= \Eb_{z\sim\mu} \int_{\Sb^{d-1}}\bigl( \langle\Gamma^{-1/2}(y-z), \theta\rangle \bigr)^2\md \theta \nonumber,
\end{align}
\end{subequations}
where, in order to move from \eqref{eq:pre_swap} to \eqref{eq:swapped_integral}, we use the fact that  $\bE_{z\sim\mu} \|z\|_{\Gamma}^2 < \infty$. Furthermore, 
\begin{align*}
    \langle \Gamma^{-1/2} a, \theta\rangle\langle\Gamma^{-1/2}b, \theta\rangle = (\Gamma^{-1/2} a)^\top \theta\theta^\top (\Gamma^{-1/2} b),
\end{align*}
and noting, from~\cite{vershynin2018high}[Section \S 3.3.1], the covariance of the uniform on the sphere, $\mathrm{ Cov \, Unif}(\Sb^{d-1})$ $= \int_{\Sb^{d-1}}\theta\theta^\top\md \theta =  \frac{1}{d}\I$.
Then
\begin{align*}
      \int_{\Sb^{d-1}} \langle \Gamma^{-1/2} a, \theta\rangle \langle \Gamma^{-1/2} b , \theta\rangle\md \theta &= \frac{1}{d}(\Gamma^{-1/2}a)^\top (\Gamma^{-1/2} b)\\
      &=\frac{1}{d}\langle a, b \rangle_\Gamma.
\end{align*}
In particular 
\begin{align*}
      \int_{\Sb^{d-1}} \bigl(\langle \Gamma^{-1/2} (y-z), \theta\rangle\bigr)^2 \md \theta
      &=\frac{1}{d}\|y-z\|^2_\Gamma.
\end{align*}
And so, from (\ref{eq:swapped_integral}) and Lemma \ref{lemma:weightedW} we obtain
\begin{align*}
    \mathsf{SW}^2_{2, \Gamma}(\delta_y, \mu) 
    &=\frac{1}{d}\Eb_{z\sim\mu}\|y-z\|^2_\Gamma\\
    &=\frac{1}{d}\mathsf{W}^2_{2, \Gamma}(\delta_y, \mu).
\end{align*}
\end{proof}

\subsection{Proof of Lemma \ref{lemma:gradient_noise_conv}}\label{proof:gradient_noise_conv}
\begin{proof}
We have
\begin{align*}
    \sI(\eta * \mu)&=\bE_{x' \sim \eta*\mu}\|y-x'\|_\Gamma^2,\\
    \sI(\mu)&=\bE_{x \sim \mu}\|y-x\|_\Gamma^2.
\end{align*}
Thus
\begin{align*}
    \sI(\eta * \mu)&=\bE_{(x,\varepsilon) \sim \mu \otimes \eta}\|y-x-\varepsilon\|_\Gamma^2,\\
   &=\bE_{(x,\varepsilon) \sim \mu \otimes \eta}\Bigl(\|y-x\|_\Gamma^2-2\langle y-x,\varepsilon\rangle_\Gamma+\|\varepsilon\|_\Gamma^2\Bigr),\\
   &=\bE_{x \sim \mu} \|y-x\|_\Gamma^2+ \bE_{\varepsilon \sim \eta}\|\varepsilon\|_\Gamma^2,\\
   &=\sI(\mu)+\bE_{\varepsilon \sim \eta}\|\varepsilon\|_\Gamma^2,
\end{align*}
using that $\eta$ is centred, independent of $\mu$ and has finite second moments.

\end{proof}

\section{Weak Form Residuals with Array Shifting}\label{app:weak_form_residuals}
We now show how to compute a residual in the weak form for a Darcy problem on domain $D$, with homogenous boundary conditions using array shifting. Taking~\eqref{eq:poisson_bc}
and testing the domain residual against as set of test functions $\{v_i\}_{i=1}^{d_o}$, we obtain the discretized variational formulation
\begin{align}
    \cO(R(z,u))_i = \langle v_i, R(z, u)  \rangle= \int_D v_i (\nabla\cdot(z\nabla u))\md x + \int_D v_i f \md x= 0, \forall v_i \in V.\nonumber
\end{align}
Through integration by parts we obtain the weak form. Here, we use as a shorthand $\r_i =  \cO(R(z,u))_i$. In 2D
\begin{align}
    \label{eq:discretized_residual}
    \r_i &= -\int_D z \partial_{x_{(1)}} v_i\,\partial_{x_{(1)}} u\,\md x -\int_D z \partial_{x_{(2)}} v_i\,\partial_{x_{(2)}} u \,\md x + \int_D v_i f \,\md x.
\end{align}
\begin{figure}[]
    \centering
    \def\l{2.4}
    \begin{tikzpicture}
      \foreach \j in {0,1,2}
        \foreach \k in {0,1,2}
          \fill (\l*\j,\l*\k) circle (3pt);
      
      \draw (\l*0,\l*0) -- (\l*0,\l*1);
      \draw (\l*0,\l*0) -- (\l*0,\l*2);
      \draw (\l*0,\l*0) -- (\l*1,\l*0);
      \draw (\l*0,\l*2) -- (\l*1,\l*2);
      \draw (\l*0,\l*1) -- (\l*1,\l*1);
      \draw (\l*1,\l*0) -- (\l*1,\l*1);
      \draw (\l*1,\l*2) -- (\l*1,\l*1);
      \draw (\l*2,\l*0) -- (\l*2,\l*1);
      \draw (\l*2,\l*0) -- (\l*2,\l*2);
      \draw (\l*2,\l*2) -- (\l*2,\l*1);
      \draw (\l*0,\l*2) -- (\l*0,\l*1);
      \draw (\l*1,\l*0) -- (\l*2,\l*0);
      \draw (\l*1,\l*1) -- (\l*2,\l*1);
      \draw (\l*1,\l*2) -- (\l*2,\l*2);
      \draw (\l*1,\l*0) -- (\l*0,\l*1);
      \draw (\l*2,\l*0) -- (\l*1,\l*1);
      \draw (\l*1,\l*1) -- (\l*0,\l*2);
      \draw (\l*2,\l*1) -- (\l*1,\l*2);

      \node at (\l*1.2,\l*1.2) {$j,k$};
      \node at (-\l*0.4,\l*1.) {$j,k_{-1}$};
      \node at (\l*2.4,\l*1.) {$j,k_{+1}$};

      \node at (\l*1.0,\l*-0.2) {$j_{+1},k$};
      \node at (\l*2.0,\l*-0.2) {$j_{+1},k_{+1}$};

      \node at (\l*0.0,\l*2.2) {$j_{-1},k_{-1}$};
      \node at (\l*1.0,\l*2.2) {$j_{-1},k_{}$};

     \node at (\l*0.2,\l*1.2) {$1$}; 
     \node at (\l*0.8,\l*0.8) {$2$};
     \node at (\l*1.2,\l*0.2) {$3$};
     \node at (\l*1.8,\l*0.8) {$4$};
     \node at (\l*1.4,\l*1.4) {$5$};
     \node at (\l*0.8,\l*1.8) {$6$};

    \end{tikzpicture}
    \caption{2D FEM mesh centred at node $jk$. We assume equal spacing, $h$, of the nodes.}
    \label{fig:2DMeshNodes}
\end{figure}
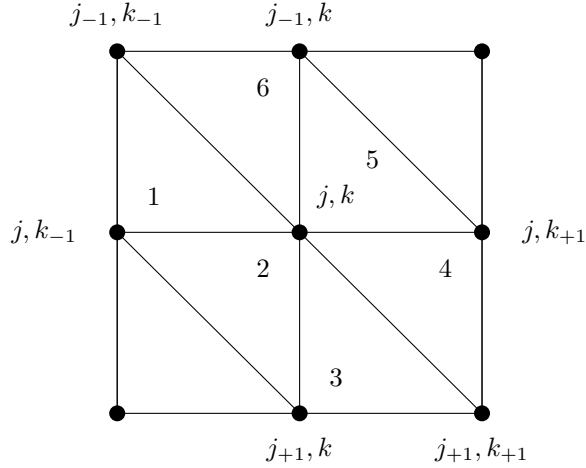
Traditional finite element solvers would assemble a system of sparse linear equations of the form $\mathsf{Au=f}$ where $\mathsf{A}\in\Rb^{d_o\times d_o}$ is sparse, and $\mathsf{u, f}\in\Rb^{d_o}$. The matrix vector product $\mathsf{Au}$ is comprised of the first two terms in~\eqref{eq:discretized_residual} and $\mathsf{f}$ is the third term. 
To compute residuals using array shifting in 2D we use double indexing, denoted by $jk$, to represent the mesh nodes as shown in Fig.~\ref{fig:2DMeshNodes}. We denote the set of indices $\{i\}_{i=1}^{d_{o}} = \mathsf{ravel}(\{j, k\}_{j,k=1}^{j, k=\sqrt{d_o}})$.
Then
\begin{align}
\r_{jk} = - \r_{x_{(1)}, jk} - \r_{x_{(2)}, jk} + \r_{f, jk},\nonumber
\end{align}
where each of the three terms corresponds  the ones in~\eqref{eq:discretized_residual}.
Assuming $z$ is given at the nodes and is piece-wise constant from the top left of an element, we have 
\begin{align}
    \r_{x_{(1)}, jk}=& \frac{1}{2} ( ( z_{j_{-1}, k_{-1}}+z_{j_{}, k_{-1}})(u_{j_{},k_{}}-u_{j_{},k_{-1}}) - (z_{j_{}, k_{}}+z_{j_{-1}, k_{}})(u_{j_{},k_{+1}}-u_{j_{},k_{}})),\nonumber\\
     \r_{x_{(2)}, jk}=& \frac{1}{2} ( (z_{j_{}, k_{-1}}+z_{j_{}, k_{}} )(u_{j_{},k_{}} - u_{j_{+1},k_{}}) - (z_{j_{-1}, k_{}}+z_{j_{-1}, k_{-1}})(u_{j_{-1},k_{}} - u_{j_{},k_{}}) ).\nonumber
\end{align}
We compute the tested inhomogeneous term as $\int_D vf\,\md x=\sum_{e=1}^6\int_{D_e}vf\,\md x$ integrating using 1 point Gauss integration at $(1/3h, 1/3h)$. 
Thus
\begin{align}
    \r_{f, jk} = \frac{h^2}{9}\left( 3 f_{j,k} +  f_{j_{-1},k_{-1}}
    + f_{j_{},k_{-1}}+ f_{j_{+1},k_{}}+ f_{j_{+1},k_{+1}}+ f_{j_{},k_{+1}}+ f_{j_{-1},k_{}} \right).\nonumber
\end{align} 
All operations can be rapidly computed using array shifting. 
To solve the linear differential equation we use an iterative linear solver such as conjugate gradient (CG), for positive definite matrices $\mathsf{A}$, or GMRES for more general problems. To solve nonlinear differential equations, one would use Newton's method. The Newton update step can be seen as solving a linear system with a Jacobian-vector product $\mathsf{u_{n+1}} = \mathsf{{u}_n - {\mathrm{J}_\cO(u_n)}^{-1}\cO(u_n) }$.
Deep learning libraries are designed to compute Jacobian-vector products with high efficiency, resulting in a GPU efficient Newton-Krilov method. 

\bibliographystyle{siamplain}
\bibliography{biblio}

\end{document}